\documentclass{article}

\usepackage{xspace,amsmath,amssymb,amsthm,amsfonts,epsfig,syntonly,dsfont,pifont}
\usepackage{bm,color,url,textcomp}
\usepackage{algorithmicx,algorithm,algpseudocode}
\usepackage{epstopdf}
\usepackage{empheq,float}
\usepackage{graphicx,subfigure,balance}
\usepackage{multirow}
\usepackage{float}
\usepackage{natbib}

\newtheorem{assumption}{Assumption}

\newtheorem{lemma}{Lemma}
\newtheorem{corollary}{Corollary}
\newtheorem{theorem}{Theorem}
\newtheorem{definition}{Definition}
\theoremstyle{definition}\newtheorem{remark}{Remark}

\DeclareMathOperator*{\argmin}{arg\,min}

\graphicspath{{./figs/}}
\usepackage{wrapfig}

\usepackage{times}
\usepackage[margin=1.05in]{geometry}

\usepackage[utf8]{inputenc} 
\usepackage[T1]{fontenc}    
\usepackage{hyperref}       
\usepackage{url}            
\usepackage{booktabs}       
\usepackage{nicefrac}       
\usepackage{microtype}      

\usepackage{scrextend}

\title{\huge{On the Convergence of SARAH and Beyond}}
\author{ 
	Bingcong Li~~~~ Meng Ma ~~~~Georgios B. Giannakis \vspace{0.1cm} \\	 
	 \textit{University of Minnesota - Twin Cities, Minneapolis, MN 55455, USA} \\
	 \texttt{\{lixx5599, maxxx971, georgios\}@umn.edu }
\vspace{0.3cm} 
}

\begin{document}

\maketitle

\begin{abstract}
The main theme of this work is a unifying algorithm, \textbf{L}oop\textbf{L}ess \textbf{S}ARAH (L2S) for problems formulated as summation of $n$ individual loss functions. L2S broadens a recently developed variance reduction method known as SARAH. To find an $\epsilon$-accurate solution, L2S enjoys a complexity of ${\cal O}\big( (n+\kappa) \ln (1/\epsilon)\big)$ for strongly convex problems. For convex problems, when adopting an $n$-dependent step size, the complexity of L2S is ${\cal O}(n+ \sqrt{n}/\epsilon)$; while for more frequently adopted $n$-independent step size, the complexity is ${\cal O}(n+ n/\epsilon)$. Distinct from SARAH, our theoretical findings support an $n$-independent step size in convex problems without extra assumptions. For nonconvex problems, the complexity of L2S is ${\cal O}(n+ \sqrt{n}/\epsilon)$. Our numerical tests on neural networks suggest that L2S can have better generalization properties than SARAH. Along with L2S, our side results include the linear convergence of the last iteration for SARAH in strongly convex problems. 
\end{abstract}

\section{Introduction}
Consider the frequently encountered empirical risk minimization (ERM) problem
\begin{align}\label{eq.prob}
	\min_{\mathbf{x} \in \mathbb{R}^d} F(\mathbf{x}) := \frac{1}{n} \sum_{i \in [n]} f_i(\mathbf{x})	
\end{align}
where $\mathbf{x} \in \mathbb{R}^d$ is the parameter to be learned from data; the set $[n]:= \{1,2,\ldots, n \}$ collects data indices; and, $f_i$ is the loss function corresponding to datum $i$. Suppose that the set of minimizers is non-empty and $F$ is bounded from below.

The standard method to solve \eqref{eq.prob} is gradient descent (GD), which per iteration $t$ relies on the update 
\begin{align*}
	\mathbf{x}_{t+1} = \mathbf{x}_{t} - \eta \nabla F(\mathbf{x}_{t})
\end{align*}
 where $\eta$ is the step size (a.k.a learning rate). For a strongly convex $F$, GD convergences linearly to $\mathbf{x}^*$, meaning after $T$ iterations it holds that $\|\mathbf{x}_T - \mathbf{x}^* \|^2 \leq c^T \|\mathbf{x}_0 - \mathbf{x}^* \|^2$ with some constant $c\in(0,1)$; while for convex $F$ it holds that $F(\mathbf{x}_T) - F(\mathbf{x}^*) = {\cal O}(1/T)$, and for nonconvex $F$ one has $\min_t \| \nabla F(\mathbf{x}_t) \| = {\cal O}(1/T)$; see e.g., \citep{nesterov2004,ghadimi2013}. However, finding $\nabla F(\mathbf{x}_t)$ per iteration in the big data regime, i.e., with large $n$, can be computationally prohibitive. To cope with this, the stochastic gradient descent (SGD) \citep{robbins1951,bottou2018} draws uniformly at random an index $i_t \in [n]$ per iteration, and updates via 
 \begin{align*}
	 \mathbf{x}_{t+1} = \mathbf{x}_{t} - \eta_t \nabla f_{i_t}(\mathbf{x}_{t}).
 \end{align*} Albeit computationally light, SGD comes with a slower convergence rate than GD \citep{bottou2018,ghadimi2013}, which is mainly due to the variance of the gradient estimate given by $\mathbb{E}[\|  \nabla f_{i_t}(\mathbf{x}_t) - \nabla F(\mathbf{x}_t) \|^2]$.

By capitalizing on the finite sum structure of ERM, a class of algorithms, variance reduction family, can be designed to solve \eqref{eq.prob} more efficiently. The idea is to judiciously (often periodically) evaluate a \textit{snapshot gradient} $\nabla F(\mathbf{x}_s)$, and use it as an anchor of the stochastic draws $\{\nabla f_{i_t}(\mathbf{x}_t)\}$ in subsequent iterations. As a result, compared with the simple gradient estimate $\nabla f_{i_t}(\mathbf{x}_t)$ in SGD, the variance of estimated gradients can be reduced. Members of the variance reduction family include SDCA \citep{shalev2013}, SVRG \citep{johnson2013,reddi2016,allen2016}, SAG \citep{roux2012}, SAGA \citep{defazio2014,reddi2016saga}, MISO \citep{mairal2013}, SCSG \citep{lei2017,lei2017non}, SNVRG \citep{zhou2018snvrg} and SARAH \citep{nguyen2017,nguyen2019}, and their variants \citep{konecny2013,kovalev2019,qian2019,li2019bb}. Most of these rely on the update $\mathbf{x}_{t+1} = \mathbf{x}_t - \eta \mathbf{v}_t$, where $\eta$ is a \textit{constant} step size and $\mathbf{v}_t$ is a carefully designed gradient estimator that takes advantage of the snapshot gradient. When aiming for an accurate solution, variance reduction methods are faster than SGD for convex and nonconvex problems, and remarkably they converge linearly when $F$ is strongly convex. The \emph{complexity} of algorithms such as GD, SGD, and variance reduction families will be quantified by the number of incremental first-order oracle (IFO) calls that counts how many (incremental) gradients are computed \citep{agarwal2014}, as specified next using our notational conventions.
\begin{definition}
An IFO is a black box with inputs $f_i$ and $\mathbf{x} \in \mathbb{R}^d$, and output the gradient $\nabla f_i(\mathbf{x})$.
\end{definition}
\noindent 
For example, the IFO complexity to compute $\nabla F(\mathbf{x})$ is $n$. For a prescribed $\epsilon$, a desirable algorithm obtains an $\epsilon$-accurate solution (defined as follows) with minimal complexity\footnote{Complexity is the abbreviation for IFO complexity throughout this work.}. 
\begin{definition}
	Let $\mathbf{x}$ be a solution returned by certain algorithm. If $\mathbb{E}[\| \nabla F(\mathbf{x}) \|^2] \leq \epsilon$ is satisfied, $\mathbf{x}$ is termed as an $\epsilon$-accurate solution to \eqref{eq.prob}.
\end{definition}

Variance reduction algorithms outperform GD in terms of complexity. And when high accuracy ($\epsilon$ small) is desired, the complexity of variance reduction methods is also lower than that of SGD. Among variance reduction algorithms, the distinct feature of SARAH \citep{nguyen2017,nguyen2019} and its variants \citep{fang2018,zhang2018,wang2018,nguyen2018,pham2019} is that they rely on a \textit{biased} gradient estimator $\mathbf{v}_t$ formed by recursively using stochastic gradients. SARAH performs comparably to SVRG/SAGA on strongly convex problems, but reduces the complexity of SVRG/SAGA for nonconvex losses. In addition, no duality (as in SDCA) or gradient table (for SAGA) is required. With SARAH's analytical and practical merits granted, there are unexplored issues. For example, guarantees on SARAH with $n$-\textit{independent} step size for convex problems are missing since analysis in \citep{nguyen2017} requires an extra presumption. The last iteration convergence of SARAH is also not well studied yet. In this context, our contributions are summarized next. 
\begin{enumerate}
	\item[\textbullet] \textbf{Unifying algorithm and novel analysis:} A new algorithm, \textbf{L}oop\textbf{L}ess \textbf{S}ARAH (L2S) is developed. It offers a unified algorithmic framework with provable convergence properties through a novel analyzing technique. To find an $\epsilon$-accurate solution, L2S enjoys a complexity ${\cal O}\big( (n+\kappa) \ln (1/\epsilon)\big)$ for strongly convex problems with condition number $\kappa$. For convex problems, the complexity of L2S is ${\cal O}(n+ \sqrt{n}/\epsilon)$ when an $n$-related step size is used; or ${\cal O}(n+ n/\epsilon)$ for an $n$-independent step size. The complexity of L2S for nonconvex problems is ${\cal O}(n+ \sqrt{n}/\epsilon)$.
	\item[\textbullet] \textbf{Tale of generalization:} Supported by experimental evidence, we find that L2S can have generalization merits compared with SARAH for nonconvex tasks such as training neural networks.
	\item[\textbullet] \textbf{Last iteration convergence of SARAH:} Linear convergence of the last iteration for SARAH on $\mu$-strongly convex problems is established. Distinct from \citep{liuclass} with step size ${\cal O}(\mu/L^2)$, our analysis enables a much larger step size, i.e., $\eta = {\cal O}(1/L)$. In addition, we find that if each $f_i$ is strongly convex, the complexity of adopting last iteration in SARAH is lower than that of SVRG.
\end{enumerate}

\textbf{Notation}. Bold lowercase letters denote column vectors; $\mathbb{E}(\mathbb{P})$ represents expectation (probability); $\| \mathbf{x}\|$ stands for the $\ell_2$-norm of a vector $\mathbf{x}$; and $\langle \mathbf{x}, \mathbf{y} \rangle$ denotes the inner product between vectors $\mathbf{x}$ and $\mathbf{y}$.

\section{Preliminaries}\label{sec.intro}
This section reviews SARAH \citep{nguyen2017,nguyen2019} with emphases on the quality of gradient estimates. Before diving into SARAH, we first state the assumptions posed on $F$ and $f_i$.

\begin{assumption}\label{as.1}
Each $f_i: \mathbb{R}^d \rightarrow \mathbb{R}$ has $L$-Lipchitz gradient, that is, $\|\nabla f_i(\mathbf{x}) - \nabla f_i(\mathbf{y}) \| \leq L \| \mathbf{x}-\mathbf{y} \|, \forall \mathbf{x}, \mathbf{y} \in \mathbb{R}^d$.
\end{assumption}
\begin{assumption}\label{as.2}
	Each $f_i: \mathbb{R}^d \rightarrow \mathbb{R}$ is convex.
\end{assumption}
\begin{assumption}\label{as.3}
	$F: \mathbb{R}^d \rightarrow \mathbb{R}$ is $\mu$-strongly convex, meaning there exists $\mu > 0$, so that $F(\mathbf{x}) - F(\mathbf{y}) \geq \langle \nabla F(\mathbf{y}), \mathbf{x}-\mathbf{y}\rangle + \frac{\mu}{2} \| \mathbf{x}-\mathbf{y}\|^2,$ $\forall \mathbf{x}, \mathbf{y} \in \mathbb{R}^d$.
\end{assumption} 
\begin{assumption}\label{as.4}
	Each $f_i: \mathbb{R}^d \rightarrow \mathbb{R}$ is $\mu$-strongly convex, meaning there exists $\mu > 0$, so that $f_i(\mathbf{x}) - f_i(\mathbf{y}) \geq \langle \nabla f_i(\mathbf{y}), \mathbf{x}-\mathbf{y}\rangle + \frac{\mu}{2} \| \mathbf{x}-\mathbf{y}\|^2,$ $\forall \mathbf{x}, \mathbf{y} \in \mathbb{R}^d$.
\end{assumption}  
Assumptions \ref{as.1} -- \ref{as.4} are standard in the analysis of variance reduction algorithms. Assumption \ref{as.1} requires each loss function to be sufficiently smooth. In fact one can distinguish the smoothness of individual loss function and refine Assumption \ref{as.1} as $f_i$ has $L_i$-Lipchitz gradient. Clearly $L =\max_i L_i$. With slight modifications on SARAH, such refinement can tighten the complexity bounds slightly. The detailed discussions can be found in Appendix \ref{sec.d2s}. In the main text, we will keep using the simpler Assumption \ref{as.1} for clarity. Assumption \ref{as.2} implies that $F$ is also convex. Assumption \ref{as.3} only requires $F$ to be strongly convex, which is slightly weaker than Assumption \ref{as.4}. And it is clear when Assumption \ref{as.4} is true, both Assumptions \ref{as.2} and \ref{as.3} hold automatically. Under Assumptions \ref{as.1} and \ref{as.3} (or \ref{as.4}), the condition number of $F$ is defined as $\kappa := L/\mu$.

\subsection{Recap of SARAH}
\begin{wrapfigure}{TL}{0.5\textwidth}
\begin{minipage}{0.5\textwidth}
\vspace{-0.8cm}
\begin{algorithm}[H]
    \caption{SARAH}\label{alg.1}
    \begin{algorithmic}[1]
    	\State \textbf{Initialize:} $\tilde{\mathbf{x}}^0 $, $\eta$, $m$, $S$
    	\For {$s=1,2,\dots,S$}
			\State $\mathbf{x}_0^s = \tilde{\mathbf{x}}^{s-1}$
			\State $\mathbf{v}_0^s =  \nabla F (\mathbf{x}_0^s )$
			\State $\mathbf{x}_1^s = \mathbf{x}_0^s - \eta \mathbf{v}_0^s$
			\For {$t=1,2,\dots,m$}
				\State Uniformly sample $i_t \in [n]$ 
				\State $\mathbf{v}_t^s = \nabla f_{i_t} (\mathbf{x}_t^s ) -\nabla f_{i_t} (\mathbf{x}_{t-1}^s ) + \mathbf{v}_{t-1}^s $
				\State $\mathbf{x}_{t+1}^s = \mathbf{x}_t^s - \eta \mathbf{v}_t^s$
			\EndFor
			\State $\tilde{\mathbf{x}}^{s}$ uniformly rnd. chosen from $\{\mathbf{x}_t^s\}_{t=0}^m$
		\EndFor
		\State \textbf{Output:} $\tilde{\mathbf{x}}^S$
	\end{algorithmic}
\end{algorithm}
\vspace{-0.3cm}
\end{minipage}
\end{wrapfigure}

\textbf{SARAH for Strongly Convex Problems}: The detailed steps of SARAH are listed under Alg. \ref{alg.1}.
In a particular outer loop (lines 3 - 11) indexed by $s$, a snapshot gradient $\mathbf{v}_0^s = \nabla F(\mathbf{x}_0^s)$ is computed first to serve as an anchor of gradient estimates $\mathbf{v}_t^s$ in the ensuing inner loop (lines 6 - 10). Then $\mathbf{x}_0^s$ is updated $m+1$ times based on $\mathbf{v}_t^s$
\begin{equation}\label{eq.?????}
	\mathbf{v}_t^s = \nabla f_{i_t} (\mathbf{x}_t^s ) -\nabla f_{i_t} (\mathbf{x}_{t-1}^s ) + \mathbf{v}_{t-1}^s.
\end{equation}
SARAH's gradient estimator $\mathbf{v}_t^s$ is \textit{biased}, since $
	\mathbb{E} \big[ \mathbf{v}_t^s| {\cal F}_{t-1} \big] \!=\! \nabla F(\mathbf{x}_t^s ) \!-\! \nabla F(\mathbf{x}_{t-1}^s ) \!+\! \mathbf{v}_{t-1}^s \!\neq\! \nabla F(\mathbf{x}_t^s )$, where ${\cal F}_{t-1}: = \sigma(\mathbf{x}_0^s, i_1, i_2, \ldots,i_{t-1})$ denotes the $\sigma$-algebra generated by $\mathbf{x}_0^s, i_1, i_2, \ldots,i_{t-1}$. Albeit biased, $\mathbf{v}_t^s$ is carefully designed to ensure the mean square error (MSE) relative to $\nabla F(\mathbf{x}_t^s)$ is bounded above, and stays proportional to $\mathbb{E}[ \|\nabla F(\tilde{\mathbf{x}}^{s-1}) \|^2]$.
\begin{lemma}\label{lemma.momt}
	\cite[Lemma 2]{nguyen2017} If Assumptions \ref{as.1} and \ref{as.2} hold and $\eta < 2/L$, SARAH guarantees that
	\begin{equation*}\label{eq.key_lemma}
		\mathbb{E} \big[ \|\nabla F(\mathbf{x}_t^s) - \mathbf{v}_t^s \|^2 \big] \leq \frac{\eta L}{2- \eta L} \mathbb{E}\big[\|\nabla F(\tilde{\mathbf{x}}^{s-1}) \|^2\big], ~\forall t.
	\end{equation*}
\end{lemma}
This MSE bound of Lemma \ref{lemma.momt} is critical for analyzing SARAH, and instrumental in establishing its linear convergence for strongly convex $F$. It is worth stressing that the step size of SARAH should be chosen by $\eta<1/L$ to ensure convergence, which can be larger than that of SVRG, whose step size should be less than $1/(4L)$.

\textbf{SARAH for Convex Problems:}
Establishing the convergence rate of SARAH with an $n$-independent step size remains open for convex problems. Regarding complexity, the only analysis implicitly assumes SARAH to be \textit{non-divergent}, as confirmed by the following claim used to derive the complexity.

\textit{Claim:} \cite[Theorem 3]{nguyen2017}
If $\delta_s := \frac{2}{\eta (m+1)} \mathbb{E} \big[F(\tilde{\mathbf{x}}^s) - F(\mathbf{x}^*) \big]$, $\delta := \max_s \delta_s$, $\Delta := \delta + \frac{\delta \eta L}{2-2\eta L}$, and $\alpha = \frac{\eta L}{2 - \eta L}$, it holds that $
		\mathbb{E}[ \| \nabla F(\tilde{\mathbf{x}}^s) \|^2 ] - \Delta \leq \alpha^s ( \| F(\tilde{\mathbf{x}}^0) \|^2 - \Delta )$.

The missing piece of this claim is that for a finite $\delta_s$ or $\delta$, $\mathbb{E} [F(\tilde{\mathbf{x}}^s) - F(\mathbf{x}^*) ]$ must be bounded; or equivalently, the algorithm must be assumed \emph{non-divergent}. Even if $\mathbb{E} [F(\tilde{\mathbf{x}}^s) - F(\mathbf{x}^*) ]$ is finite, assuming it to be ${\cal O}(1)$ as in \citep{nguyen2017} is not reasonable. Another variant of SARAH in \citep{nguyen2018inexact} also relies on a similar assumption to guarantee convergence. We will show that the proposed algorithm can bypass this extra non-divergent assumption.

\textbf{SARAH for Nonconvex Problems.}
SARAH also works for nonconvex problems if line 11 in Alg. \ref{alg.1} is modified to $\tilde{\mathbf{x}}^s = \mathbf{x}_{m+1}^s$. The key to convergence again lies in the MSE of $\mathbf{v}_t^s$.
\begin{lemma}\label{lemma.sarah_nc}
\cite[Lemma 1]{fang2018} If Assumption \ref{as.1} holds, the MSE of $\mathbf{v}_t^s$ is bounded by
	\begin{align*}
		\mathbb{E} \big[  \|\nabla F(\mathbf{x}_t^s) -  \mathbf{v}_t^s \|^2  \big]  \leq  \eta^2 L^2  \sum_{\tau=0}^{t-1} \mathbb{E} \big[ \| \mathbf{v}_{\tau}^s  \|^2 	\big].
	\end{align*}
\end{lemma}

Lemma \ref{lemma.sarah_nc} states that the upper bound of MSE of $\mathbf{v}_t^s$ is i) proportional to $\eta^2$; and, ii) larger when $t$ is larger. Leveraging the MSE bound, it was established that the complexity to find an $\epsilon$-accurate solution is ${\cal O}(n+ \sqrt{n}/\epsilon)$ \citep{nguyen2019}. Compared with SARAH, the proposed algorithm has its own merits for tasks such as training neural network, which will be clear in Section \ref{sec.discussion}.

\section{Loopless SARAH}
This section presents the \textbf{L}oop\textbf{L}ess \textbf{S}ARAH (L2S) algorithmic framework, which is capable of dealing with (strongly) convex and nonconcex ERM problems.

\begin{algorithm}[t]
    \caption{L2S}\label{alg.2}
    \begin{algorithmic}[1]
    	\State \textbf{Initialize:} $\mathbf{x}_0 $, $\eta$, $m$, $T$
    	\State Compute $\mathbf{v}_0 = \nabla F(\mathbf{x}_0) $ \hfill \Comment{Compute a snapshot gradient}
   		\State $\mathbf{x}_1 = \mathbf{x}_0 - \eta \mathbf{v}_0$
        	\For {$t=1,2,\dots,T$}
			\State Choosing $\mathbf{v}_t$ via \hfill \Comment{A randomized snapshot gradient scheduling}
			\begin{equation*}\label{eq.l2s_v}
				\mathbf{v}_t = \left\{
    				\begin{array}{ll}
    				{\nabla F(\mathbf{x}_t)}~ &\text{w.p.} ~\frac{1}{m}
         				\\
         			{\nabla f_{i_t}(\mathbf{x}_t) - \nabla f_{i_t}(\mathbf{x}_{t-1}) + \mathbf{v}_{t-1} }~  & \text{w.p.} ~1 - \frac{1}{m}
    				\end{array}
   				\right.
 			\end{equation*}
			\State $\mathbf{x}_{t+1} = \mathbf{x}_t - \eta \mathbf{v}_t$ 
		\EndFor
		\State \textbf{Output:} uniformly chosen from $\{\mathbf{x}_t\}_{t=1}^T$
	\end{algorithmic}
\end{algorithm}

L2S is summarized in Alg. \ref{alg.2}. Besides the single loop structure, the most distinct feature of L2S is that $\mathbf{v}_t$ is a probabilistically computed snapshot gradient given by
	\begin{equation}\label{eq.l2s_v}
		\mathbf{v}_t = \left\{
    		\begin{array}{ll}
    		{\!\nabla F(\mathbf{x}_t)} &\text{w.p.~} 1/m 		\\
         	{\!\nabla f_{i_t}(\mathbf{x}_t) - \nabla f_{i_t}(\mathbf{x}_{t\!-\!1}) + \mathbf{v}_{t\!-\!1} }  & \text{w.p.~} 1 \!-\! 1/m
    		\end{array}
   		\right.
	\end{equation}
where $i_t \in [n]$ is again uniformly sampled. The gradient estimator $\mathbf{v}_t$ is still biased, since $\mathbb{E}[\mathbf{v}_t| {\cal F}_{t-1}]$ $= \nabla F(\mathbf{x}_t) - (1 \!-\! \frac{1}{m}) \big[\nabla F(\mathbf{x}_{t-1} )- \mathbf{v}_{t-1}^s \big]\neq \nabla F(\mathbf{x}_t )$. In L2S, the snapshot gradient is computed every $m$ iterations \textit{in expectation}, while SARAH computes the snapshot gradient once every $m+1$ updates. The emergent challenge is that one has to ensure a small MSE of $\mathbf{v}_t$ to guarantee convergence, where the difficulty arises from the randomness of when a snapshot gradient is computed.

An equivalent manner to describe \eqref{eq.l2s_v} is through a sequence of i.i.d. Bernoulli random variables $\{B_t\}$ with pmf
\begin{align}\label{eq.Bt}
	\mathbb{P}(B_t = 1) = \frac{1}{m};~~ \mathbb{P}(B_t = 0)= 1- \frac{1}{m}.
\end{align}
If $B_t=1$, a snapshot gradient $\mathbf{v}_t = \nabla F(\mathbf{x}_t)$ is computed; otherwise, the estimated gradient $\mathbf{v}_t = \nabla f_{i_t}(\mathbf{x}_t) - \nabla f_{i_t}(\mathbf{x}_{t-1}) +  \mathbf{v}_{t-1} $ is used for the update. Let $N_{t_1:t}$ denote the event that at iteration $t$ the last evaluated snapshot gradient was at $t_1$. In other words, $N_{t_1:t}$ is equivalent to $B_{t_1} = 1, B_{t_1+1} = 0, \ldots, B_t = 0$. Note that $t_1$ can take values from $0$ (no snapshot gradient computed) to $t$ (corresponding to $\mathbf{v}_t = \nabla F(\mathbf{x}_t)$). 

The key lemma enabling our analysis is a simple probabilistic observation.
\begin{lemma}\label{lemma.p_N}
	For a given $t$, i) events $N_{t_1:t}$ and $N_{t_2:t}$ are disjoint when $t_1 \neq t_2$; and, ii) $\sum_{t_1 = 0}^t \mathbb{P}( N_{t_1:t} ) = 1$.
\end{lemma}

The general idea is to exploit these properties of $N_{t_1:t}$ to obtain the MSE of $\mathbf{v}_t^s$, which is further leveraged to derive the convergence of L2S. Note that our idea for establishing the convergence of L2S is general enough to provide a parallel analysis for a loopless version of SVRG \citep{kovalev2019,qian2019}, without relying on the complicated Lyapunov function.

\subsection{L2S for Convex Problems}\label{sec.l2s}
The subject of this subsection is problems with smooth and convex losses such as those obeying Assumptions \ref{as.1} and \ref{as.2}. We find that SARAH is challenged analytically because $\tilde{\mathbf{x}}^s \neq \mathbf{x}_{m+1}^s$ in Line 11 of Alg. \ref{alg.1}, which necessitates SARAH's `non-divergent' assumption. A few works have identified this issue \citep{nguyen2019,wang2018,pham2019}, but require an $n$-dependent step size (e.g., $\eta = {\cal O}(\frac{1}{L\sqrt{n}})$) to address it\footnote{These algorithms are designed for nonconvex problems, however, even assuming convexity we are unable to show the convergence with a step size independent with $n$.}. However, $n$-independent step sizes are also widely adopted in practice. The key to bypassing this $n$-dependence in step size, is removing the inner loop of SARAH and computing snapshot gradients following a random schedule as \eqref{eq.l2s_v}.

The analysis starts with the MSE of $\mathbf{v}_t$ in L2S. All proofs are relegated to Appendix due to space limitations.
\begin{lemma}\label{lemma.l2s}
	Under Assumptions \ref{as.1} and \ref{as.2}, the following inequality holds for a given $t$ when $\eta < 2/L$
	\begin{align}\label{eq.l2s.cond}
		 \mathbb{E} \big[ \|\nabla & F(\mathbf{x}_t)- \mathbf{v}_t \|^2  | N_{t_1:t} \big] \leq \frac{ \eta L}{2 - \eta L} \mathbb{E} \big[ \|\nabla F(\mathbf{x}_{t_1}) \|^2  \big].
	\end{align}
	Furthermore, we have
	\begin{align*}
		 \mathbb{E} \big[  \|\nabla & F(\mathbf{x}_t)- \mathbf{v}_t \|^2  \big]  \leq \frac{ \eta L}{2 - \eta L}   \Big(1-\frac{1}{m}\Big)^t \|\nabla F(\mathbf{x}_0) \|^2 + \frac{ \eta L}{2 - \eta L}\frac{1}{m} \sum_{\tau=1}^{t-1}  \Big(1-\frac{1}{m}\Big)^{t-\tau} \mathbb{E}\big[ \|\nabla F(\mathbf{x}_{\tau}) \|^2 \big]  .
	\end{align*}
\end{lemma}
Comparing \eqref{eq.l2s.cond} with Lemma \ref{lemma.momt} reveals that conditioning on $N_{t_1:t}$, $\mathbf{x}_{t_1}$ in L2S is similar to the starting point of an outer loop in SARAH (i.e., $\mathbf{x}_0^s$), while the following iterations $\{\mathbf{x}_\tau\}_{\tau = t_1+1}^t$ mimic the behavior of SARAH's inner loop. Taking expectation w.r.t. $N_{t_1:t}$ in \eqref{eq.l2s.cond}, Lemma \ref{lemma.l2s} further asserts that the MSE of $\mathbf{v}_t$ depends on the \textit{exponentially moving average} of the norm square of past gradients. 

\begin{theorem}\label{thm.l2s}
	If Assumptions \ref{as.1} and \ref{as.2} hold, and the step size is chosen such that $\eta <1/L$ and $1 - \frac{ \eta L}{2 - \eta L} \geq C_\eta $, where $C_\eta$ is a positive constant, the output of L2S, $\mathbf{x}_a$, is guaranteed to satisfy
	\begin{align*}
		\mathbb{E} & \big[ \| \nabla  F(\mathbf{x}_a) \|^2\big] = {\cal O}\bigg( \frac{ F(\mathbf{x}_0) - F(\mathbf{x}^* )}{\eta TC_\eta}  + \frac{m \eta L \| \nabla F(\mathbf{x}_0) \|^2 }{T C_\eta} \bigg).
	\end{align*}
\end{theorem}
The constant $C_\eta$ depends on the choice of $\eta$, e.g., $C_\eta=2/3$ for $\eta = 0.5/L$. Based on Theorem \ref{thm.l2s}, the convergence rates as well as the complexities under different choices of $\eta$ and $m$ are specified in the following corollaries. Let us start with a constant step size that is irrelevant with $n$. 
\begin{corollary}\label{coro.l2s1.1}
	Choose a constant $\eta<1/L$. If $m = \Theta(\sqrt{n})$, then L2S has convergence rate ${\cal O}(\sqrt{n}/T)$ and requires ${\cal O}(n+ n/\epsilon)$ IFO calls to find an $\epsilon$-accurate solution.
\end{corollary}
\begin{corollary}\label{coro.l2s1.2}
	Choose a constant $\eta<1/L$. If $m = \Theta(n)$, the convergence rate of L2S is ${\cal O}(n/T)$. The complexity to ensure an $\epsilon$-accurate solution is ${\cal O}(n+ n/\epsilon)$.
\end{corollary}
In Corollaries \ref{coro.l2s1.1} and \ref{coro.l2s1.2}, the choice of $\eta$ does not depend on $n$. Thus, relative to SARAH, L2S eliminates the non-divergence assumption and establishes the convergence rate as well. 
On the other hand, an $n$-dependent step size is also supported by L2S, whose complexity is specified in the following corollary.
\begin{corollary}\label{coro.l2s2}
If we select $\eta = {\cal O}\big( \frac{1}{L\sqrt{m}}\big)$, and $m = \Theta(n)$, then L2S has convergence rate ${\cal O}(\sqrt{n}/T)$, and the complexity to find an $\epsilon$-accurate solution is ${\cal O}(n+ \sqrt{n}/\epsilon)$. 
\end{corollary}

\textbf{When to Adopt $n$-dependent Step Sizes?}
An interesting observation is that though the complexity of using an $n$-dependent step size is lower than those of an $n$-independent step size in both L2S and SVRG \citep{reddi2016}, the numerical performances on modern datasets such as \textit{rcv1} and \textit{a9a} suggest that $n$-independent step sizes boost the convergence speed. We argue that an $n$-dependent step size only reveals its numerical merits when $n$ is \textit{extremely large}. Intuitively, a large $n$ positively correlates with the larger MSE of the gradient estimate, which in turn calls for a smaller ($n$-dependent) step size. Our numerical results in Appendix \ref{apdx.when} also support this argument. We subsample aforementioned datasets with different values of $n$. SVRG and L2S are tested on these subsampled datasets. Besides the faster convergence when using an $n$-independent step size, it is also observed that as $n$ increases, i) the gradient norm of solutions obtained by $n$-dependent step sizes becomes smaller; and ii) the difference on the performance gap between $n$-dependent and $n$-independent step sizes reduces.

\subsection{L2S for Nonconvex Problems}\label{sec.l2s_nc}
The scope of L2S can also be broadened to nonconvex problems under Assumption \ref{as.1}, that is, L2S with a proper step size is guaranteed to use ${\cal O}(n+ \sqrt{n}/\epsilon)$ IFO calls to find an $\epsilon$-accurate solution. Compared with SARAH, the merit of L2S is that the extra MSE introduced by the randomized scheduling of snapshot gradient computation can be helpful for exploring the landscape of the loss function, which will be discussed in detail in Section \ref{sec.discussion}. Here we focus on the convergence properties only, starting with the MSE in nonconvex settings.
\begin{lemma}\label{lemma.l2s_nc}
	If Assumption \ref{as.1} holds, L2S guarantees that for a given $N_{t_1:t}$
	\begin{align}\label{eq.est_error}
		 \mathbb{E} \big[ \|\nabla & F(\mathbf{x}_t)  - \mathbf{v}_t \|^2 | N_{t_1:t} \big]  \leq \eta^2 L^2 \sum_{\tau = t_1 +1}^{t}  \mathbb{E}\big[ \| \mathbf{v}_{\tau -1} \|^2 | N_{t_1:t}  \big].
	\end{align}
	In addition, the following inequality is true
	\begin{align*}
		\mathbb{E} \big[  \|\nabla F(\mathbf{x}_t) -  \mathbf{v}_t \|^2  \big]  \leq  \eta^2 L^2  \sum_{\tau=0}^{t-1} \bigg( 1 - \frac{1}{m} \bigg)^{t - \tau} \mathbb{E} \big[ \| \mathbf{v}_{\tau}  \|^2 	\big].
	\end{align*}
\end{lemma}
Conditioning on $ N_{t_1:t}$, iterations $\{\mathbf{x}_\tau\}_{\tau = t_1}^t$ are comparable to an outer loop of SARAH. Similar to Lemma \ref{lemma.sarah_nc}, the MSE upper bound of $\mathbf{v}_t$ in \eqref{eq.est_error} is large when $t-t_1$ is large. If we take expectation w.r.t. the randomness of $ N_{t_1:t}$, the MSE of $\mathbf{v}_t $ then depends on the exponentially moving average of the norm square of all past gradient estimates $\{\mathbf{v}_{\tau }\}_{\tau = 0}^{t-1}$, which is different from  Lemma \ref{lemma.l2s} (for convex problems) where the MSE involves the past gradients $\{\nabla F(\mathbf{x}_{\tau})\}_{\tau = 0}^{t-1}$. It turns out that such a past-estimate-based MSE is difficult to cope with using only the exponentially deceasing sequence $\{(1-1/m)^{t-\tau}\}_{\tau = 0}^{t-1}$, prompting a cautiously designed ($m$-dependent) $\eta$. 

\begin{theorem}\label{thm.l2s_nc}
With Assumption \ref{as.1} holding, and choosing $\eta \in (0,  \frac{\sqrt{4m+1} - 1}{2mL}] = {\cal O}\big( \frac{1}{L \sqrt{m}} \big)$, the final L2S output $\mathbf{x}_a$ satisfies
	\begin{align*}
	 	& ~~~ \mathbb{E} \Big[ \|\nabla F(\mathbf{x}_a)\|^2 \Big] = {\cal O} \bigg( \frac{ L \sqrt{m} \big[ F(\mathbf{x}_0) - F(\mathbf{x}^* )\big]}{ T} + \frac{ \| \nabla F(\mathbf{x}_0)  \|^2 }{T} \bigg).
	 \end{align*}
\end{theorem}
An intuitive explanation of the $m$-dependent $\eta$ is that with a small $m$, L2S evaluates a snapshot gradient more frequently [cf. \eqref{eq.l2s_v}], which translates to a relatively small MSE bound in Lemma \ref{lemma.l2s_nc}. Given an accurate gradient estimate, it is thus reasonable to adopt a larger step size.

\begin{corollary}\label{coro.l2s_nc}
Selecting $\eta = {\cal O}\big( \frac{1}{L\sqrt{m}}\big)$ and $m = \Theta(n)$, L2S converges with rate ${\cal O}(\sqrt{n}/T)$, and the complexity to find an $\epsilon$-accurate solution is ${\cal O}(n+ \sqrt{n}/\epsilon)$. 
\end{corollary}
Almost matching the lower bound ${\Omega}(\sqrt{n}/\epsilon)$ of nonconvex ERM problems \citep{fang2018}, the complexity of L2S is similar to other SARAH type algorithms \citep{fang2018,wang2018,nguyen2019}. The slight suboptimality is due to the $n$ extra IFO calls involved in computing $\mathbf{v}_0$.

\subsection{L2S for Strongly Convex Problems}\label{sec.l2s-sc}
In addition to convex and nonconvex problems, a modified version of L2S that we term \textbf{L2S} for \textbf{S}trongly \textbf{C}onvex problems (L2S-SC), converges linearly under Assumptions \ref{as.1} -- \ref{as.3}. As we have seen previously, L2S is closely related to SARAH, especially when conditioned on a given $N_{t_1:t}$. Hence, we will first state a useful property of SARAH that will guide the design and analysis of L2S-SC.

\begin{algorithm}[t]
    \caption{L2S-SC}\label{alg.l2s_sc}
    \begin{algorithmic}[1]
    	\State \textbf{Initialize:} $\mathbf{x}_0 $, $\eta$, $m$, $S$, and $s=0$
    	\State Compute $\mathbf{v}_0 = \nabla F(\mathbf{x}_0) $ \hfill \Comment{Compute a snapshot gradient}
   		\State $\mathbf{x}_1 = \mathbf{x}_0 - \eta \mathbf{v}_0$
        	\While {$s \neq S$}
			\State Randomly generate $B_t$ as \eqref{eq.Bt} \hfill \Comment{$\mathbf{v}_t$ is computed equivalent to \eqref{eq.l2s_v}}
			\If {$B_t = 1$} 
			\State $\mathbf{x}_t = \mathbf{x}_{t-1}$ \hfill \Comment{Step back when a snapshot gradient is computed}
			\State  $\mathbf{v}_t = \nabla F(\mathbf{x}_t)$, ~$s = s+1$
			\Else
				\State $\mathbf{v}_t = \nabla f_{i_t}(\mathbf{x}_t) - \nabla f_{i_t}(\mathbf{x}_{t-1}) +  \mathbf{v}_{t-1} $
			\EndIf
			\State $\mathbf{x}_{t+1} = \mathbf{x}_t - \eta \mathbf{v}_{t}$, ~$t = t+1$
		\EndWhile
		\State $T = t$
		\State \textbf{Output:} $\mathbf{x}_{T}$
	\end{algorithmic}
\end{algorithm}

\begin{lemma}\label{lemma.sarah_sc}
	Consider SARAH (Alg. \ref{alg.1}) with Line 11 replaced by $\tilde{\mathbf{x}}^{s} = \mathbf{x}_m^{s}$. 	Choosing $\eta < 2/(3L)$ and $m$ large enough such that
	\begin{align*}
		\lambda_m := \frac{2 \eta L}{2- \eta L} + \big( 2 + 2\eta L \big) (\theta)^m  < 1,
	\end{align*}
	where $\theta$ is defined as 
	\begin{equation}\label{eq.theta}
				\theta = \left\{
    				\begin{array}{ll}
    				{1- \Big(\frac{2}{\eta L} -1 \Big)\mu^2 \eta^2} &\text{with As. \ref{as.1} -- \ref{as.3}}          				\\
         			{ 1 - \frac{2 \eta L}{1 + \kappa} }  & \text{with As. \ref{as.1} and \ref{as.4}}      		
         			\end{array}.
   				\right.
 			\end{equation}
	The modified SARAH is guaranteed to converge linearly; that is, 
	\begin{align*}
		\mathbb{E} \big[\|\nabla F(\tilde{\mathbf{x}}^s)\|^2 \big] \leq \lambda_m \mathbb{E}\big[ \|\nabla F(\tilde{\mathbf{x}}^{s-1}) \|^2 \big].
	\end{align*}
\end{lemma}
As opposed to the random draw of $\tilde{\mathbf{x}}^{s}$ (Line 11 of Alg. \ref{alg.1}), Lemma \ref{lemma.sarah_sc} asserts that by properly choosing $\eta$ and $m$, setting $\tilde{\mathbf{x}}^{s} = \mathbf{x}_m^{s}$ preserves the linear convergence of SARAH. Note that the convergence with last iteration of SARAH was also studied by \citep{liuclass} under Assumptions \ref{as.1} -- \ref{as.3}. However, their analysis requires an undesirably small step size, i.e., $\eta = {\cal O}(\mu/L^2)$, while ours enables a much larger one $\eta = {\cal O}(1/L)$.

\begin{remark}
Through Lemma \ref{lemma.sarah_sc} one can establish the complexity of SARAH with $\tilde{\mathbf{x}}^{s} = \mathbf{x}_m^{s}$. When Assumptions \ref{as.1} - \ref{as.3} hold, the complexity is ${\cal O}\big( (n+\kappa^2) \ln \frac{1}{\epsilon}\big)$, which is on the same order of SVRG with last iteration \citep{tan2016,hu2018diss}. However, when Assumptions \ref{as.1} and \ref{as.4} are true, the complexity of SARAH decreases to ${\cal O}\big( (n+\kappa) \ln \frac{1}{\epsilon}\big)$. This is the property SVRG does not exhibit.
\end{remark}

L2S-SC is summarized in Alg. \ref{alg.l2s_sc}, where $\mathbf{v}_t$ obtained in Lines 5 - 11 is a rewrite of \eqref{eq.l2s_v} using $B_t$ introduced in \eqref{eq.Bt} for the ease of presentation and analysis. 
L2S-SC differs from L2S in that when $B_t =1$, $\mathbf{x}_t$ steps back slightly as in Line 7. This "step back" is to allow for a rigorous analysis, and can be viewed as the counterpart of choosing $\tilde{\mathbf{x}}^s = \mathbf{x}_{m}^s$ instead of $\mathbf{x}_{m+1}^s$ as in Lemma \ref{lemma.sarah_sc}. Omitting Line 7 in practice does not deteriorate performance. In addition, the parameter $S$ required to initialize L2S is comparable to the number of outer loops of SARAH, as one can also validate through the $S$ dependence in the linear convergence rate.
\begin{theorem}\label{thm.l2s_sc}
	Choose $\eta< 2/(3L)$ and $m$ large enough such that
	\begin{align*}
		\lambda := \frac{2 \eta L}{2- \eta L} + \frac{ 2 + 2\eta L}{m-1}  	\frac{\theta (1- \frac{1}{m})}{1 - \theta (1-\frac{1}{m})}	< 1
	\end{align*}
	where $\theta$ is defined in \eqref{eq.theta}. L2S-SC in Alg. \ref{alg.l2s_sc} guarantees
	\begin{align*}
		\mathbb{E} \big[\|\nabla F(\mathbf{x}_T) \|^2 \big] \leq   \lambda^S \|\nabla F(\mathbf{x}_0) \|^2.
	\end{align*}
\end{theorem}

The complexities of L2S-SC under different assumptions are established in the next corollaries.
\begin{corollary}\label{coro.l2s_nc}
Choose $\eta < 2/(3L)$ and $m = \Theta(\kappa^2)$. When Assumptions \ref{as.1} -- \ref{as.3} hold, the complexity of L2S to find an $\epsilon$-accurate solution is ${\cal O}\big( (n+\kappa^2) \ln \frac{1}{\epsilon}\big)$. \end{corollary}
\begin{corollary}\label{coro.l2s_nc.1}
Choose $\eta < 2/(3L)$ with $m = {\Theta}(\kappa)$. When Assumptions \ref{as.1} and \ref{as.4} hold, the complexity of L2S to find an $\epsilon$-accurate solution is ${\cal O}\big( (n+\kappa) \ln \frac{1}{\epsilon}\big)$.
\end{corollary}

\section{Discussions}\label{sec.discussion}

\subsection{Comparison with SCSG}
L2S can be viewed as SARAH with variable inner loop length. A variant of SVRG (abbreviated as SCSG) with randomized inner loop length has been also developed in \citep{lei2017,lei2019}. A close relative of SCSG is a loopless version of SVRG \citep{kovalev2019,qian2019}. Unfortunately, the analysis in \citep{kovalev2019} is confined to strongly convex problems, while \citep{qian2019} relies on different analyzing schemes that are more involved. The key differences between L2S and SCSG are as follows. 

d1) The random inner loop length of SCSG is assumed geometrically distributed (at least for the analysis) that could be even infinite. Thus, its total number of iterations is also random. In contrast, the total number of L2S iterations is fixed to $T+1$. This is accomplished through a non-geometrically distributed equivalent inner loop length.

d2) The analyses are also different. From a high level, SCSG employs a ``forward'' analysis, where an iteration $t$ that computes a snapshot gradient is fixed first, and then future iterations $t+1$, $t+2$ till the computation of the next snapshot gradient are checked; while our analysis takes the ``backward'' route, that is, after fixing an iteration $t$ the past iterations $t-1, t-2, \ldots, 0$ are checked for a snapshot gradient computation. As a consequence, our ``backward'' analysis leads to an exponentially moving average structure in the MSE (Lemmas \ref{lemma.l2s} and \ref{lemma.l2s_nc}), which is insightful, and is not provided by SCSG.
 

\subsection{Generalization Merits of L2S}

\begin{wrapfigure}{r}{0.45\textwidth}
\vspace{-0.35cm}
	\includegraphics[height=3.8cm]{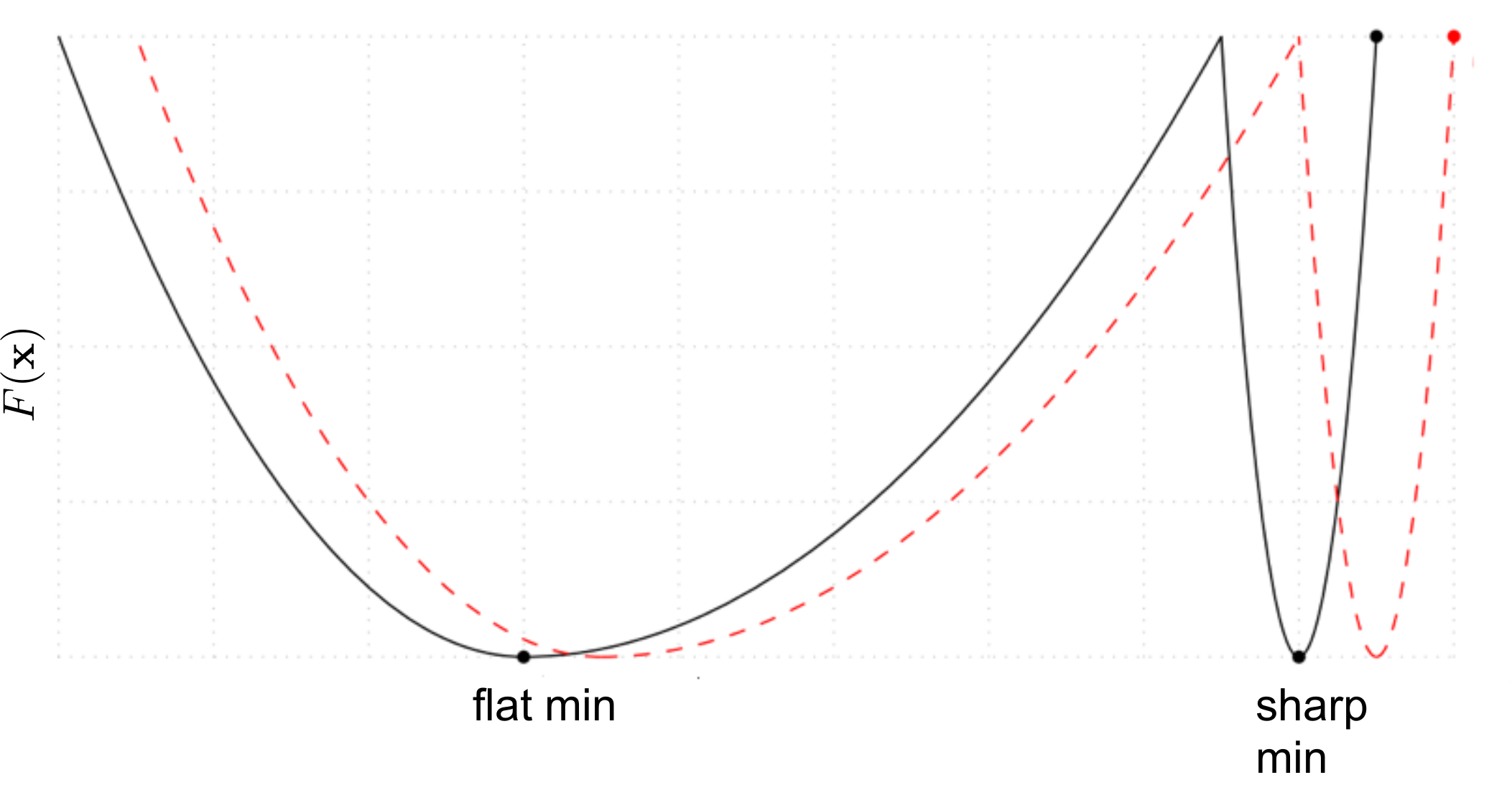}
	\vspace{-0.38cm}
	\caption{An illustration of sharp and flat minima \citep{keskar2016}. The black line is the loss curvature associated with training data; and the red line represents the loss for testing data which slightly deviates from the training loss. The sensitivity of the training function at a sharp minimum degrades its generalization performance.}
	\label{fig.nnnn}
\end{wrapfigure}

SARAH has well-documented merits for its complexity for nonconvex problems, but similar to other variance reduction algorithms, it is not as successful as expected for training neural networks. We conjecture this is related with the reduced MSE of the gradient estimates. To see this, although there is no consensus on analytical justification for this, empirical evidence points out that SGD with large mini-batch size (needed to reduce the variance of the gradient estimates) tends to converge to a sharp minimum. Sharp minima are believed to have worse generalization properties compared with flat ones \citep{keskar2016}. Fig. \ref{fig.nnnn} shows that gradient estimators with pronounced variability are more agile to explore the space and escape from a \textit{sharp minimum}, while flat minima are more tolerant to larger variability. This suggests that the gradient estimator could be designed to control its exploration ability, which can in turn improve generalization. Being able to explore the loss function landscape is critical because deepening and widening a neural network does not always endow the stochastic gradient estimator with controllable exploration ability \citep{defazio2018}.

These empirical results shed light on the important role of exploration ability in the gradient estimates. A natural means of controlling this exploration in algorithms with variance reduction, is to add zero-mean noise in the gradient estimates. However, the issue is that even for convex problems, the convergence rate slows down if the noise is not carefully calibrated; see e.g. \citep{kulunchakov2019}. Carefully designed noises for \textit{escaping saddle points} rather than \textit{generalization merits} were studied in e.g., \citep{jin2017,fang2019}. But even for saddle point escaping, extra information of the loss landscape, e.g., Hessian Lipchitz constant is required for obtaining the variance of injected noise. Unfortunately, the Hessian Lipchitz constant is not always available in advance. To control the exploration ability of algorithms with variance reduction, L2S resorts to randomized snapshot gradient computation that is free of extra knowledge for the loss landscape.

With SARAH as a reference, we can see how our randomized snapshot gradient computation in L2S can benefit variability for exploration. Let $t_2 - t_1$ denote the equivalent length of a L2S inner loop, where $t_1$ and $t_2$ are the indices of two consecutive iterations when snapshot gradients are computed. Recall from Lemma \ref{lemma.sarah_nc} that the MSE of $\mathbf{v}_t$ tends to be larger as $t$ approaches $t_2$. Relative to SARAH, this means that the randomized computation of the snapshot gradient increases the MSE when it so happens that $t_1 + m < t < t_2$.

\begin{figure*}[t]
	\centering
	\begin{tabular}{cc}
		\includegraphics[width=.43\textwidth]{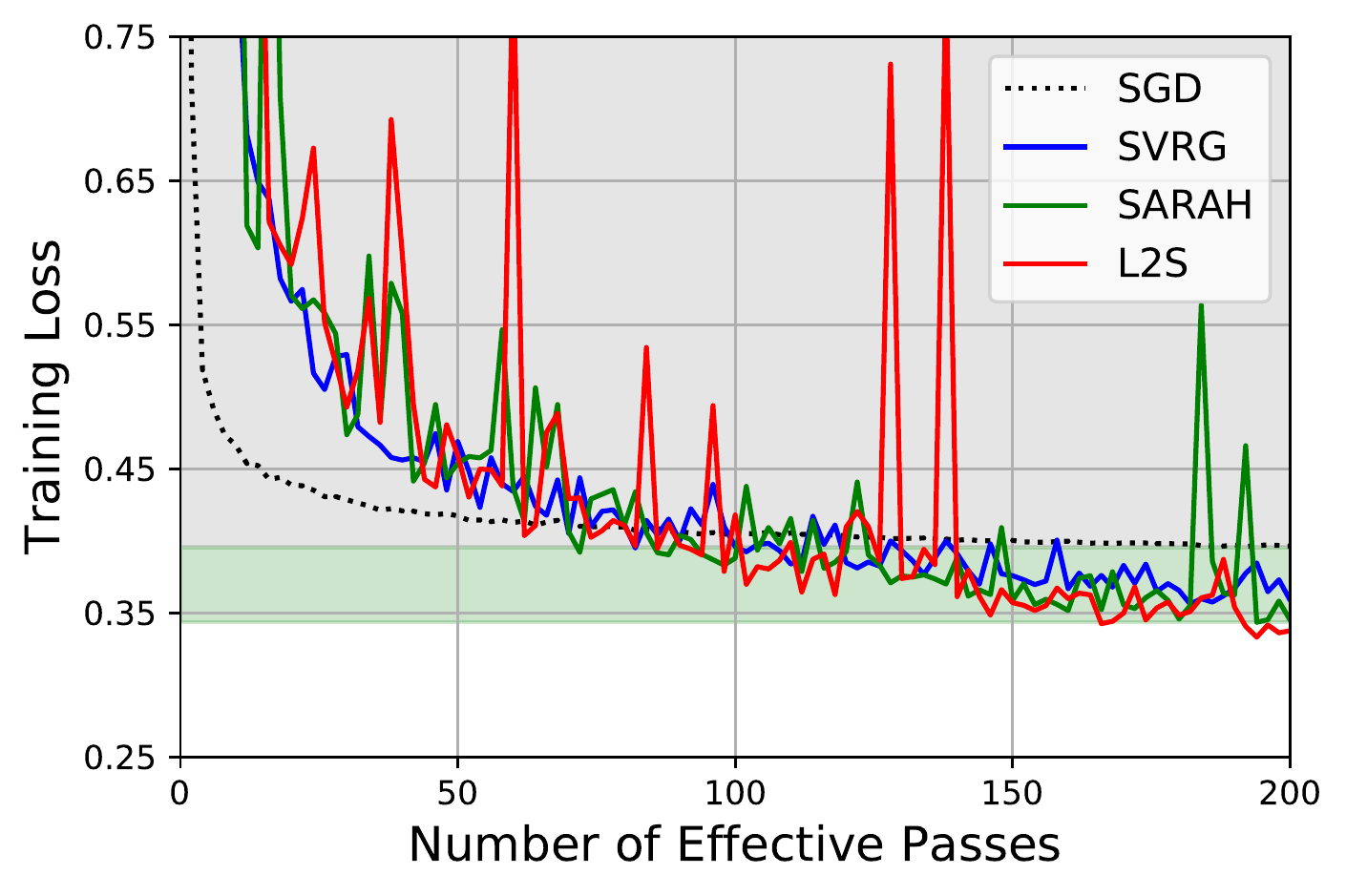}&
		\includegraphics[width=.43\textwidth]{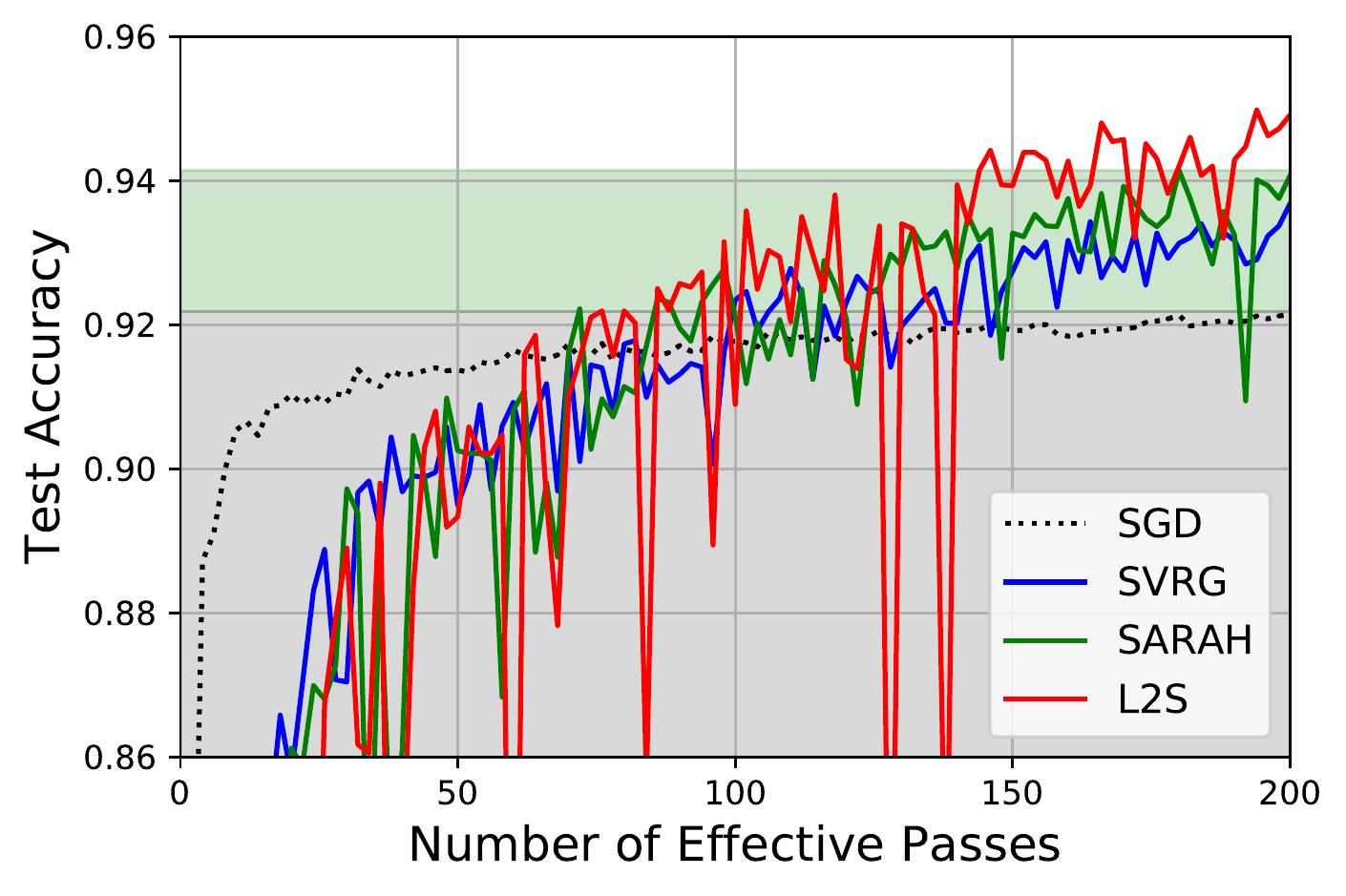}\\
		(a)  & (b)
	\end{tabular}
	\caption{(a) Training loss of L2S; (b) Test accuracy of L2S.}
	 \label{fig.nn}
\end{figure*}

\textbf{Test of L2S on Neural Networks.} We perform classification on MNIST dataset\footnote{Online available at \url{http://yann.lecun.com/exdb/mnist/}} using a $ 784 \times 128 \times 10$ feedforward neural network with sigmoid activation function. The network is trained for $200$ epochs and the training loss and test accuracy are plotted in Fig. \ref{fig.nn}. The bound of gray shadowed area indicates the smallest training loss (highest test accuracy) of SGD, while the bound of green shadowed area represents the best performances for SARAH. Figs. \ref{fig.nn} (a) and (b) share some common patterns: i) SGD converges much faster in the initial phase compared with variance reduced algorithms; ii) the fluctuate of L2S is larger than that of SARAH, implying the randomized full gradient computation indeed introduces extra chances for exploration; and, iii) when x-axis is around $140$, L2S begins to outperform SARAH while in previous epochs their performances are comparable. Note that \textit{before} L2S outperforms SARAH, there is a deep drop on its accuracy. This can be explained as that L2S explores for a local minimum with generalization merits thanks to the randomized snapshot gradient computation and the deep drop in Fig. \ref{fig.nn} (b) indicates the transition from a local min to another.

\section{Numerical Tests}
Besides training neural networks, we also apply L2S to logistic regression to showcase the performances in strongly convex and convex cases. Specifically, consider the loss function 
\begin{equation}\label{eq.test}
	F(\mathbf{x}) =\frac{1}{n} \sum_{i \in [n]} \ln \big[1+ \exp(- b_i \langle \mathbf{a}_i, \mathbf{x} \rangle ) \big] + \frac{\lambda}{2} \| \mathbf{x} \|^2
\end{equation}
where $(\mathbf{a}_i, b_i)$ is the (feature, label) pair of datum $i$. Problem \eqref{eq.test} can be rewritten in the format of \eqref{eq.prob} with $f_i(\mathbf{x}) =  \ln \big[1+ \exp(- b_i \langle \mathbf{a}_i, \mathbf{x} \rangle ) \big] + \frac{\lambda}{2} \| \mathbf{x} \|^2$. One can verify that in this case Assumptions \ref{as.1} and \ref{as.4} are satisfied. Datasets \textit{a9a}, \textit{w7a} and \textit{rcv1.binary}\footnote{All datasets are from LIBSVM, which is online available at \url{https://www.csie.ntu.edu.tw/~cjlin/libsvmtools/datasets/binary.html}.} are used in numerical tests presented. Details regarding the datasets and implementation are deferred to Appendix \ref{apdx.tests}.

\textbf{Test of L2S-SC on Strongly Convex Problems.} The performance of L2S-SC is shown in the first row of Fig. \ref{fig.l2s_sim}. SVRG, SARAH and SGD are chosen as benchmarks, where SGD is with modified step size $\eta_k = 1/\big(L(k+1)\big)$ on the $k$-th effective sample pass. For both SARAH and SVRG, the length of inner loop is chosen as $m = n$. We tune step size and only report the best performance. For a fair comparison we set $\eta$ and $m$ for L2S the same as SARAH. It can be seen that on datasets \textit{w7a} and \textit{rcv1} L2S-SC has comparable performances with SARAH, while on dataset \textit{a9a}, L2S-SC has similar performance with SARAH. The simulations validate the theoretical results of L2S-SC. 

\textbf{Test of L2S on Convex Problems.} The performances of L2S for convex problems ($\lambda =0$) is listed in the second row of Fig. \ref{fig.l2s_sim}. Again SVRG, SARAH and SGD are adopted as benchmarks. We choose $m=n$ for SVRG, SARAH and L2S. It can be seen that on datasets \textit{w7a} and \textit{rcv1} L2S performs almost the same as SARAH, while outperforms SARAH on dataset \textit{a9a}. Note that the performance of SVRG improves over SARAH on certain datasets. This is because a theoretically unsupported step size ($\eta > 1/(4L)$) is used in SVRG for best empirical performance.

\begin{figure*}[t]
	\centering
	\begin{tabular}{ccc}
		\hspace{-0.5cm}
		\includegraphics[width=.33\textwidth]{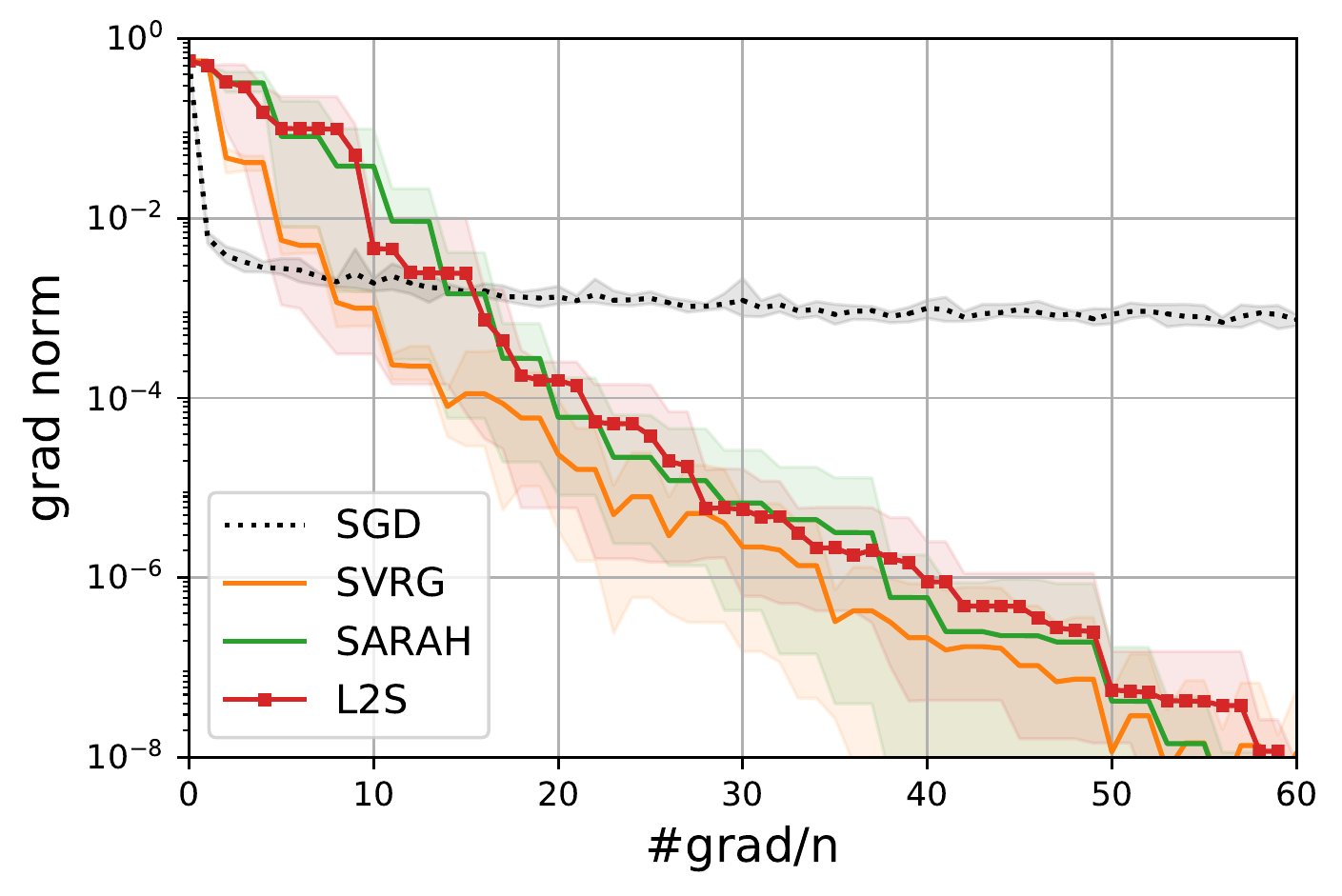}&
		\hspace{-0.5cm}
		\includegraphics[width=.33\textwidth]{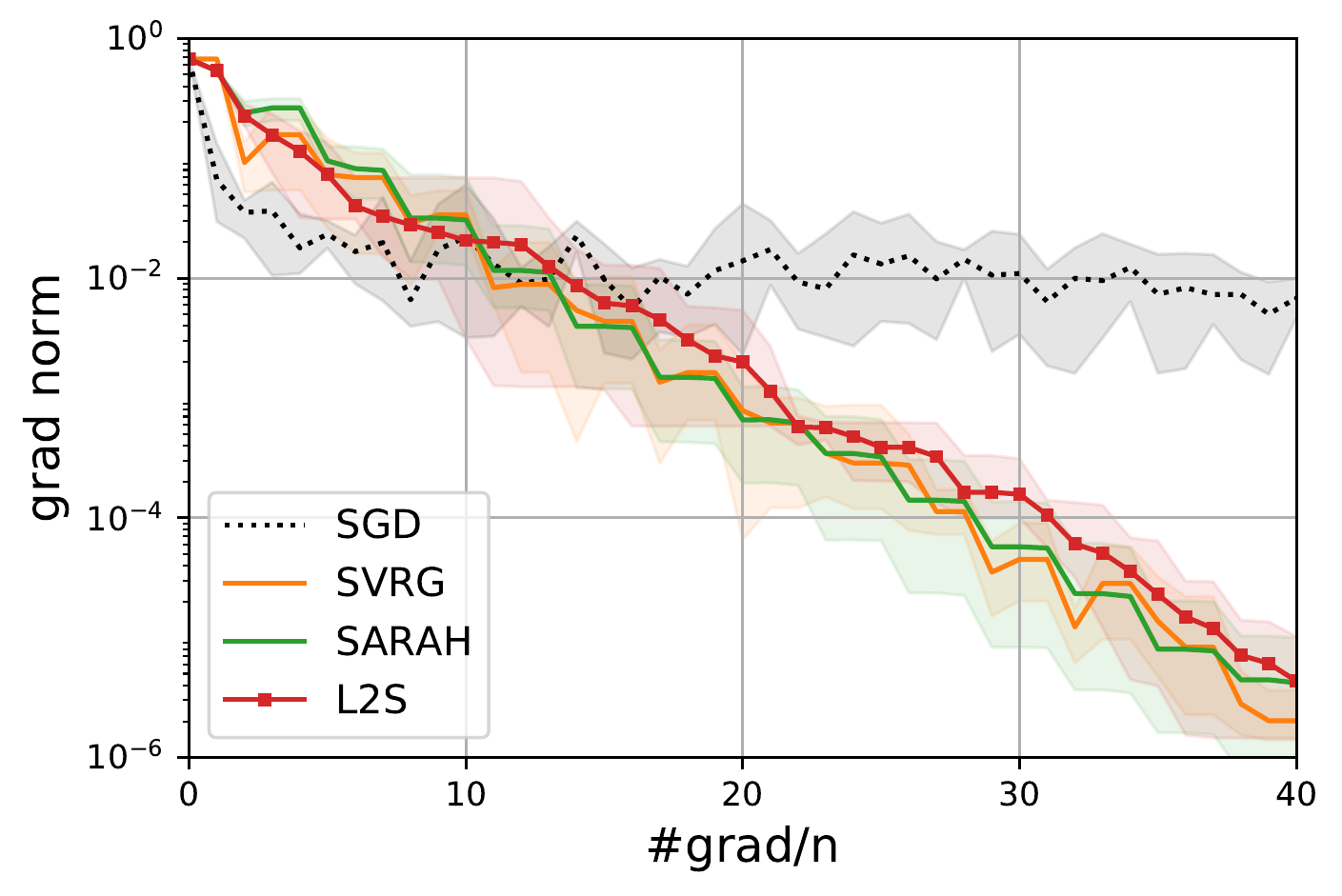}&
		\hspace{-0.5cm}
		\includegraphics[width=.33\textwidth]{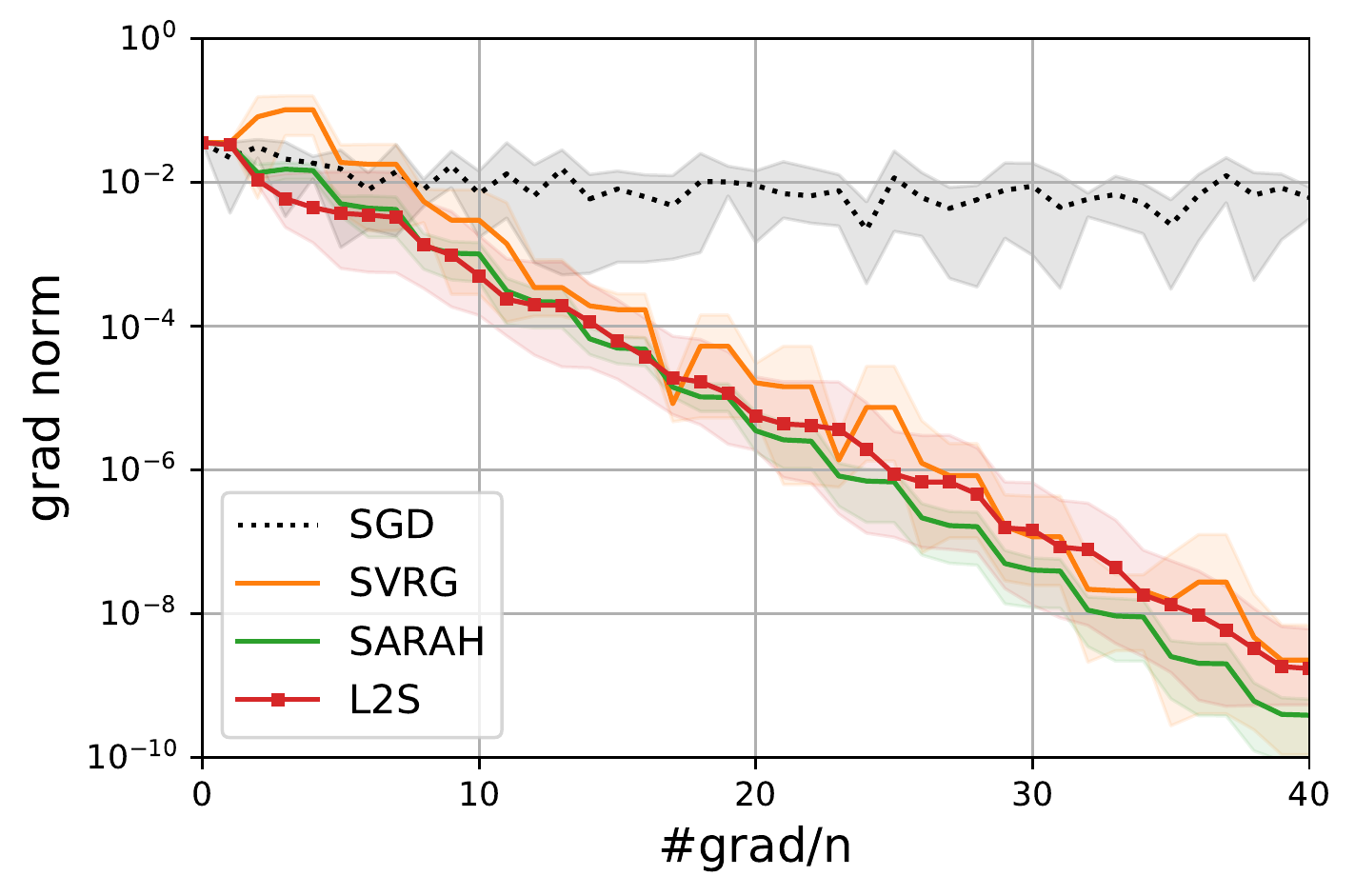}
		\\
		\hspace{-0.5cm}
		\includegraphics[width=.33\textwidth]{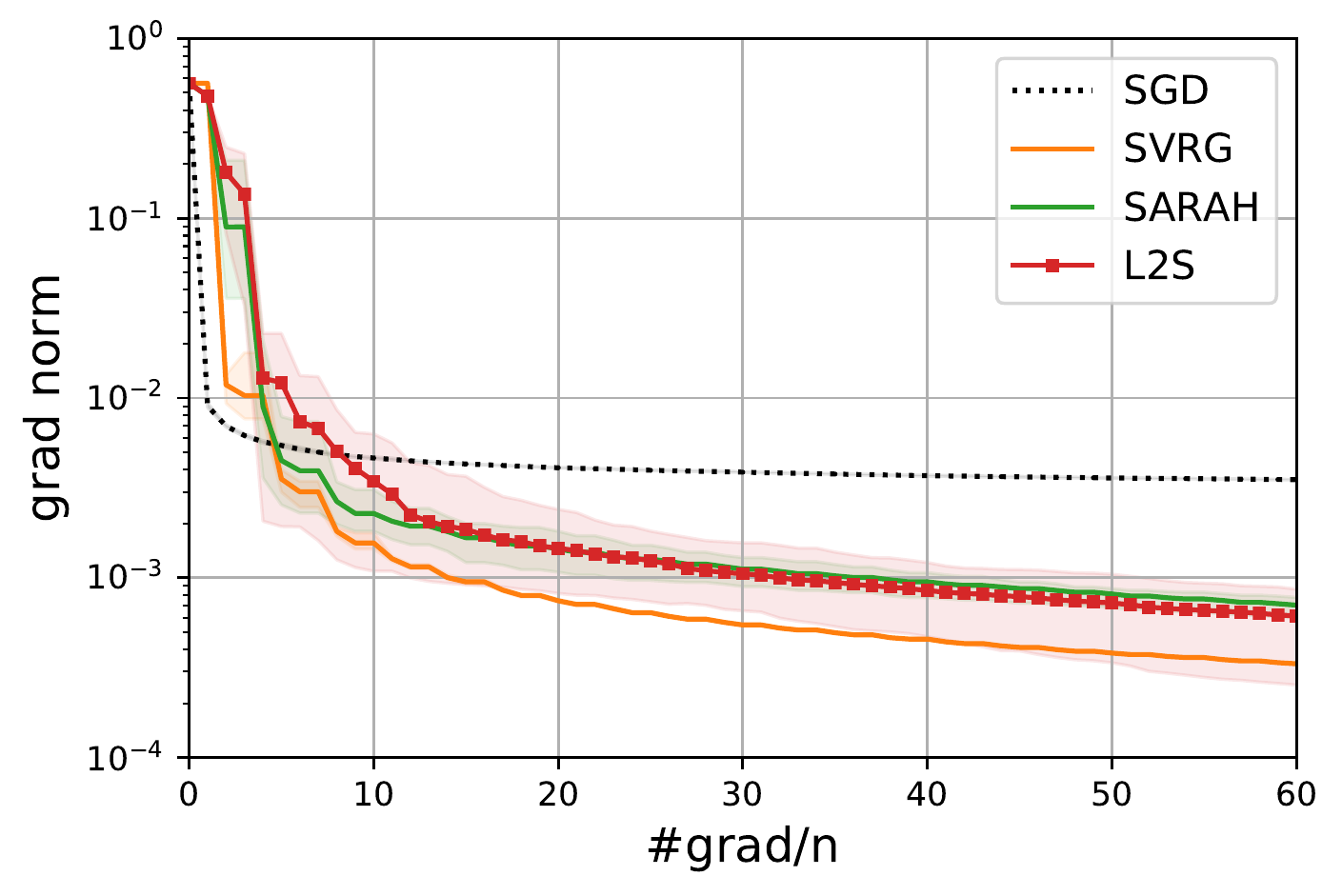}&
		\hspace{-0.5cm}
		\includegraphics[width=.33\textwidth]{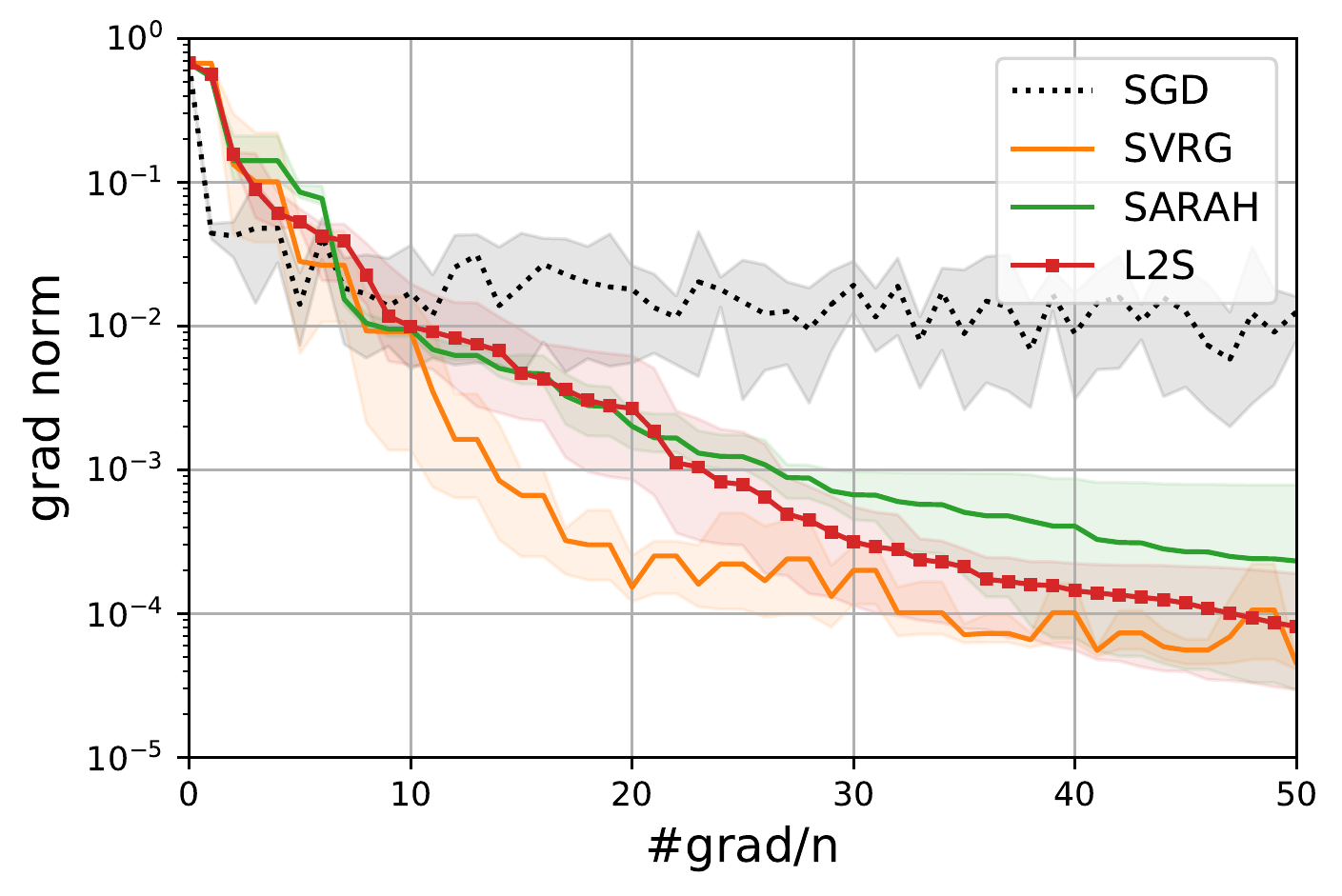}&
		\hspace{-0.5cm}
		\includegraphics[width=.33\textwidth]{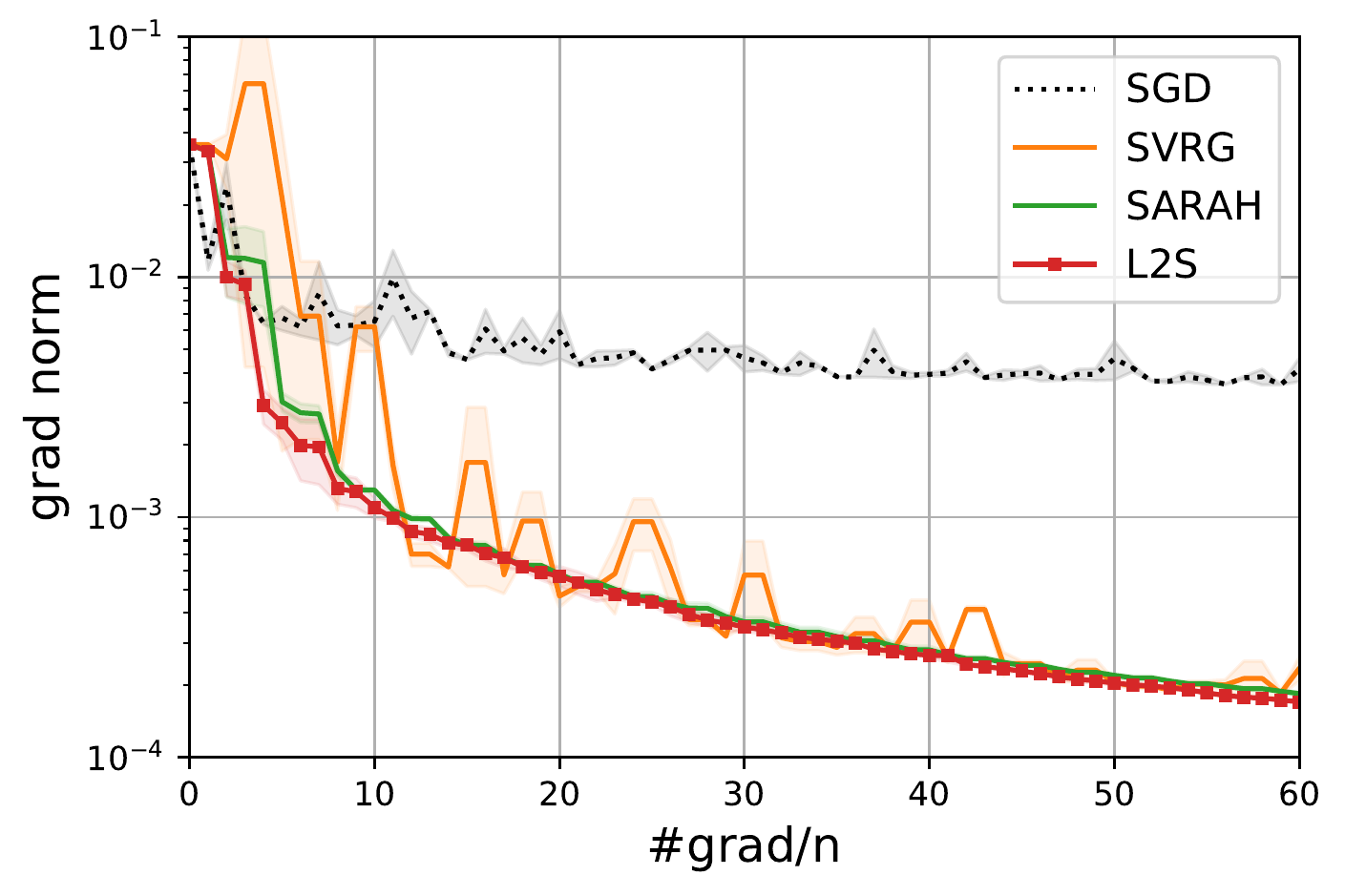}
		\\
		(a) \textit{w7a} & (b) \textit{a9a} & (c) \textit{rcv1} 
	\end{tabular}
	\caption{Tests of L2S on strongly convex problems (first row) and convex ones (second row) on different datasets.}
	 \label{fig.l2s_sim}
\end{figure*}

\section{Conclusions and Future Work}
A unifying framework, L2S, was introduced to efficiently solve (strongly) convex and nonconvex ERM problems. The complexities to find an $\epsilon$-accurate solution were established. Numerical tests validated our theoretical findings.

An interesting question is how to extend L2S and SARAH to stochastic optimization, i.e., solving $\min_{\mathbf{x}}\mathbb{E}_{\xi}[f(\mathbf{x},\xi)]$, where $\xi$ is a random variable whose distribution is unknown. Such problems can be addressed using SVRG or SCSG; see e.g., \citep{lei2017non}. Works such as \citep{nguyen2018inexact} is the first attempt for solving stochastic optimization via SARAH. Though addressing certain challenges, the remaining issue is that the gradient estimate is in general not implementable on problems other than ERM. To see this, recall the SARAH based gradient estimate for stochastic optimization is $\mathbf{v}_t = \nabla f(\mathbf{x}_t, \xi_t) -  \nabla f(\mathbf{x}_{t-1}, \xi_t)	+ \mathbf{v}_{t-1}$, where $\xi_t$ is the $t$-th realization of $\xi$. Obtaining $\nabla f(\mathbf{x}_{t-1},\xi_t)$ via a stochastic first order oracle can be impossible especially when $\xi$ comes from an unknown continuous probability space. To overcome this challenge is included in our research agenda.

\bibliographystyle{plainnat}
\bibliography{myabrv,datactr}

\newpage
\onecolumn
\appendix
\begin{center}
{\huge \bf Appendix}
\end{center}

\section{Useful Lemmas and Facts}
\begin{lemma}
\cite[Theorem 2.1.5]{nesterov2004}. If $f$ is convex and has $L$-Lipschitz gradient, then the following inequalities are true
  \begin{subequations}
	\begin{equation}\label{eq.apdx.smooth.0}
		f(\mathbf{x}) - f(\mathbf{y}) \leq \langle \nabla f(\mathbf{y}), \mathbf{x} - \mathbf{y} \rangle + \frac{L}{2} \| \mathbf{x} - \mathbf{y}\|^2
	\end{equation}
	\begin{equation}\label{eq.apdx.smooth.1}
		f(\mathbf{x}) - f(\mathbf{y}) \geq \langle \nabla f(\mathbf{y}), \mathbf{x} - \mathbf{y} \rangle  + \frac{1}{2L} \big\| \nabla f(\mathbf{x}) - \nabla f(\mathbf{y}) \big\|^2
	\end{equation}
	\begin{equation}\label{eq.apdx.smooth.2}
		\langle \nabla f(\mathbf{x}) - \nabla f(\mathbf{y}), \mathbf{x} - \mathbf{y} \rangle   \geq \frac{1}{L} \big\| \nabla f(\mathbf{x}) - \nabla f(\mathbf{y}) \big\|^2.
	\end{equation}
  \end{subequations}
\end{lemma}
Note that inequality \eqref{eq.apdx.smooth.0} does not require the convexity of $f$.

\begin{lemma}\label{lemma.nestrov.sc}
\citep{nesterov2004}. If $f$ is $\mu$-strongly convex and has $L$-Lipschitz gradient, with $\mathbf{x}^* := \argmin_{\mathbf{x}} f(\mathbf{x})$, the following inequalities are true
  \begin{subequations}
	\begin{equation} \label{eq.sca}
		2 \mu \big( f(\mathbf{x}) - f(\mathbf{x}^*) \big) \leq \| \nabla f(\mathbf{x})\|^2 \leq 2L \big( f (\mathbf{x}) - f (\mathbf{x^*}) \big)
	\end{equation}
	\begin{equation}\label{eq.scb}
		 \mu \| \mathbf{x} - \mathbf{x}^* \| \leq \| \nabla f(\mathbf{x})\| \leq L \| \mathbf{x} - \mathbf{x}^* \|
	\end{equation}
	\begin{equation} \label{eq.scc}
		\frac{\mu}{2} \| \mathbf{x} - \mathbf{x}^* \|^2 \leq f(\mathbf{x}) - f(\mathbf{x}^*) \leq \frac{L}{2} \| \mathbf{x} - \mathbf{x}^* \|^2
	\end{equation}
	\begin{equation} \label{eq.scd}
		\big\langle \nabla f(\mathbf{x}) - \nabla f(\mathbf{y}), \mathbf{x}-\mathbf{y} \big\rangle \geq \mu \|\mathbf{x}-\mathbf{y} \|^2.
	\end{equation}
  \end{subequations}
\end{lemma}
\begin{proof}
	By definition $ f(\mathbf{x}^*) -  f(\mathbf{x}) \geq \langle \nabla f(\mathbf{x}), \mathbf{x}^* - \mathbf{x} \rangle + \frac{\mu}{2} \| \mathbf{x} - \mathbf{x}^* \|$,  minimizing over $\mathbf{x} - \mathbf{x}^*$ on the RHS results in \eqref{eq.sca}. Inequality \eqref{eq.scb} follows from \cite[Theorem 2.1.9]{nesterov2004} and the fact $\nabla f(\mathbf{x}^*) = 0$.
	Inequality \eqref{eq.scc} is from \cite[Theorem 2.1.7]{nesterov2004}; and, \eqref{eq.scd} is from \cite[Theorem 2.1.9]{nesterov2004}
\end{proof}

\noindent\textbf{Proof of Lemma \ref{lemma.p_N}:} \\
\begin{proof}
	If $t_1 \neq t_2$, $N_{t_1:t}$ and $N_{t_2:t}$ are disjoint by definition, since the most recent calculated snapshot gradient can only appear at either $t_1$ or $t_2$. Since $\{B_t\}$ are i.i.d., one can find the probability of $N_{t_1:t}$ as
	\begin{equation}\label{eq.p_N}
	\mathbb{P}(N_{t_1:t}) = \left\{
    \begin{array}{ll}
    	{\frac{1}{m} \big(1-\frac{1}{m}\big)^{t-t_1}}~ &\text{if $1 \leq t_1 \leq t$}
    	         \\
    	         \\
         {\big(1-\frac{1}{m}\big)^t }~  & \text{if $t_1 = 0$}.
    \end{array}
   \right.
	\end{equation}	
	Hence one can verify that
	\begin{align*}
		\sum_{t_1 = 0}^t & \mathbb{P}(N_{t_1:t})  = \Big(1-\frac{1}{m}\Big)^t + \sum_{t_1=1}^{t-1} \frac{1}{m} \big(1-\frac{1}{m}\big)^{t-t_1} + \frac{1}{m}  \Big(1-\frac{1}{m}\Big)^t + \frac{1}{m}\frac{ 1 - \frac{1}{m} -(1-\frac{1}{m})^t}{1- (1-\frac{1}{m})} +  \frac{1}{m} = 1
	\end{align*}
	which completes the proof. 
\end{proof}

\section{Technical Proofs in Section \ref{sec.l2s}}
\subsection{Proof of Lemma \ref{lemma.l2s}}
The following lemmas are needed for the proof.
\begin{lemma}\label{lemma.copy2}
The following equation is true for $t>t_1$
	\begin{align*}
		 \mathbb{E} \big[ \|\nabla F(\mathbf{x}_t)- \mathbf{v}_t \|^2 | N_{t_1:t} \big] = & \sum_{\tau=t_1+1}^{t}\mathbb{E} \big[ \|  \mathbf{v}_\tau  - \mathbf{v}_{\tau-1} \|^2 | N_{t_1:t}	\big]  -  \sum_{\tau=t_1+1}^t \mathbb{E} \big[ \|  \nabla F(\mathbf{x}_\tau)   -  \nabla F(\mathbf{x}_{\tau-1}) \|^2 | N_{t_1:t}	\big].
	\end{align*}	
\end{lemma}
\begin{proof}
Consider that
	\begin{align}\label{eq.14}
		& \mathbb{E}\big[ \| \nabla F(\mathbf{x}_t) - \mathbf{v}_t \|^2 | {\cal F}_{t-1}, N_{t_1:t} \big] \nonumber \\
		= & \mathbb{E}\big[ \| \nabla F(\mathbf{x}_t) - \nabla F(\mathbf{x}_{t-1})+\nabla F(\mathbf{x}_{t-1}) -\mathbf{v}_{t-1}+ \mathbf{v}_{t-1}- \mathbf{v}_t \|^2 | {\cal F}_{t-1}, N_{t_1:t} \big] \nonumber \\
		= & \| \nabla F(\mathbf{x}_t) - \nabla F(\mathbf{x}_{t-1})\|^2 + \mathbb{E}\big[\| \mathbf{v}_t - \mathbf{v}_{t-1} \|^2 | {\cal F}_{t-1}, N_{t_1:t} \big] + \| \nabla F(\mathbf{x}_{t-1}) - \mathbf{v}_{t-1}\|^2  \nonumber \\
		~~~~~ & + 2\big\langle  \nabla F(\mathbf{x}_t) - \nabla F(\mathbf{x}_{t-1}),  \nabla F(\mathbf{x}_{t-1}) -\mathbf{v}_{t-1} \big\rangle  \nonumber \\
		~~~~~ &+ 2 \mathbb{E}\big[\big\langle  \nabla F(\mathbf{x}_t) - \nabla F(\mathbf{x}_{t-1}),  \mathbf{v}_{t-1}- \mathbf{v}_t \big\rangle| {\cal F}_{t-1}, N_{t_1:t} \big] \nonumber \\
		~~~~~ & + 2\mathbb{E}\big[ \big\langle  \nabla F(\mathbf{x}_{t-1}) -\mathbf{v}_{t-1} ,  \mathbf{v}_{t-1}- \mathbf{v}_t \big\rangle| {\cal F}_{t-1}, N_{t_1:t} \big] \nonumber \\
		= & \mathbb{E}\big[\| \mathbf{v}_t - \mathbf{v}_{t-1} \|^2 | {\cal F}_{t-1}, N_{t_1:t} \big] - \| \nabla F(\mathbf{x}_t) - \nabla F(\mathbf{x}_{t-1})\|^2 + \| \nabla F(\mathbf{x}_{t-1}) - \mathbf{v}_{t-1}\|^2
	\end{align}
	where the last equation is because $\mathbb{E}[\mathbf{v}_t - \mathbf{v}_{t-1} | {\cal F}_{t-1}, N_{t_1:t}] = \nabla F(\mathbf{x}_t) - \nabla F(\mathbf{x}_{t-1})$. We can expand $\mathbb{E}[\| \nabla F(\mathbf{x}_{t-1}) - \mathbf{v}_{t-1}\|^2| {\cal F}_{t-2}, N_{t_1:t}] $ using the same argument.
	Note that we have $\nabla F(\mathbf{x}_{t_1}) = \mathbf{v}_{t_1}$, which suggests
	\begin{align*}
		\mathbb{E}\big[ \| \nabla F(\mathbf{x}_{t_1+1}) - \mathbf{v}_{t_1+1} \|^2 | {\cal F}_{t_1}, N_{t_1:t} \big] 
		= \mathbb{E}\big[\| \mathbf{v}_{t_1+1} \!-\! \mathbf{v}_{t_1} \|^2 | {\cal F}_{t_1}, N_{t_1:t} \big] \!-\! \| \nabla F(\mathbf{x}_{t_1+1}) - \nabla F(\mathbf{x}_{t_1})\|^2.
	\end{align*}
	Then taking expectation w.r.t. ${\cal F}_{t-1}$ and expanding $\mathbb{E}[\| \nabla F(\mathbf{x}_{t-1}) - \mathbf{v}_{t-1}\|^2]$ in \eqref{eq.14}, the proof is completed.
\end{proof}

\textbf{Proof of Lemma \ref{lemma.l2s}:}
The implication of this Lemma \ref{lemma.p_N} is that \textit{law of total probability} \citep{gubner2006} holds. Specifically, for a random variable $C_t$ that happens in iteration $t$, the following equation holds
\begin{align}\label{eq.13}
	\mathbb{E}\big[C_t\big] = \sum_{t_1=0}^t	 \mathbb{E}\big[C_t | N_{t_1:t}\big] \mathbb{P}\{ N_{t_1:t} \}.
\end{align}

Now we turn to prove Lemma \ref{lemma.l2s}. To start with, consider that when $t_1 \neq t$
\begin{align*}
		& \quad ~ ~ \mathbb{E} \big[ \| \mathbf{v}_t  \|^2 | {\cal F}_{t-1}, N_{t_1:t} \big]  = \mathbb{E} \big[ \| \mathbf{v}_t -  \mathbf{v}_{t-1}  + \mathbf{v}_{t-1} \|^2 | {\cal F}_{t-1},N_{t_1:t} \big] \nonumber \\
		& = \| \mathbf{v}_{t-1} \|^2 + \mathbb{E} \big[ \|  \mathbf{v}_t -  \mathbf{v}_{t-1} \|^2 | {\cal F}_{t-1},N_{t_1:t} \big] + 2   \mathbb{E} \big[ \langle  \mathbf{v}_{t-1}  ,\mathbf{v}_t -  \mathbf{v}_{t-1} \rangle | {\cal F}_{t-1}, N_{t_1:t} \big] \nonumber \\
		& \stackrel{(a)}{=} \| \mathbf{v}_{t-1} \|^2 + \mathbb{E} \Big[ \|  \mathbf{v}_t -  \mathbf{v}_{t-1} \|^2  + \frac{2}{\eta}  \Big\langle  \mathbf{x}_{t-1} -  \mathbf{x}_t , \nabla f_{i_t} (\mathbf{x}_t ) -\nabla f_{i_t} (\mathbf{x}_{t-1} )\Big\rangle \Big| {\cal F}_{t-1},N_{t_1:t} \Big] \nonumber \\
		& \stackrel{(b)}{\leq} \| \mathbf{v}_{t-1} \|^2 + \mathbb{E} \Big[ \|  \mathbf{v}_t -  \mathbf{v}_{t-1}\|^2 - \frac{2}{\eta L }  \| \nabla f_{i_t} (\mathbf{x}_t ) -\nabla f_{i_t} (\mathbf{x}_{t-1} )  \|^2 \Big| {\cal F}_{t-1},N_{t_1:t} \Big] \nonumber \\
		& = \| \mathbf{v}_{t-1} \|^2 + \mathbb{E} \Big[ \|  \mathbf{v}_t -  \mathbf{v}_{t-1}\|^2  - \frac{2}{\eta L }  \|  \mathbf{v}_t -  \mathbf{v}_{t-1} \|^2 \Big| {\cal F}_{t-1},N_{t_1:t} \Big] \nonumber \\
		&  = \| \mathbf{v}_{t-1} \|^2 + \mathbb{E} \Big[  \Big(1 - \frac{2}{\eta L} \Big)\|  \mathbf{v}_t -  \mathbf{v}_{t-1} \|^2 \Big| {\cal F}_{t-1},N_{t_1:t} \Big]
	\end{align*}
where (a) follows from \eqref{eq.?????} and the update $\mathbf{x}_t = \mathbf{x}_{t-1} - \eta \mathbf{v}_{t-1} $; and (b) is the result of \eqref{eq.apdx.smooth.2}. Then by choosing $\eta$ such that $1 - \frac{2}{\eta L} < 0$, i.e., $\eta < 2/L$, we have 
	\begin{align}\label{eq.sub1}
		 \mathbb{E} \big[  \big\|  \mathbf{v}_t -  \mathbf{v}_{t-1}\|^2 | {\cal F}_{t-1},N_{t_1:t}  \big] \leq \frac{ \eta L}{2 - \eta L}  \bigg( \|\mathbf{v}_{t-1} \|^2 -  \mathbb{E} \big[ \|\mathbf{v}_t \|^2 | {\cal F}_{t-1} ,N_{t_1:t}  \big]  \bigg).
	\end{align}
	Plugging \eqref{eq.sub1} into Lemma \ref{lemma.copy2}, we have
	\begin{align*}
		 \mathbb{E} \big[ \|\nabla F(\mathbf{x}_t)- \mathbf{v}_t \|^2 &|{\cal F}_{t_1-1}, N_{t_1:t} \big] \leq\sum_{\tau=t_1+1}^{t}\mathbb{E} \big[ \|  \mathbf{v}_\tau  - \mathbf{v}_{\tau-1} \|^2 | {\cal F}_{t_1-1}, N_{t_1:t}	\big] \nonumber \\
		 & =  \frac{ \eta L}{2 - \eta L} \mathbb{E} \big[ \|\mathbf{v}_{t_1} \|^2 | {\cal F}_{t_1-1}, N_{t_1:t} \big] = \frac{ \eta L}{2 - \eta L}  \|\nabla F (\mathbf{x}_{t_1}) \|^2 
	\end{align*}
	where the last equation is because conditioning on $N_{t_1:t}$, $ \mathbf{v}_{t_1} = \nabla F (\mathbf{x}_{t_1})$.
	Note that when $t_1 = t$, this inequality automatically holds since the LHS equals to $0$. Because the randomness of $\nabla F(\mathbf{x}_{t_1})$ is irrelevant to $B_{t_1}$ (thus $N_{t_1:t}$), after taking expectation w.r.t. ${\cal F}_{t_1-1}$, we have 
	\begin{align*}
		 \mathbb{E} \big[ \|\nabla F(\mathbf{x}_t)- \mathbf{v}_t \|^2 &| N_{t_1:t} \big] \leq \frac{ \eta L}{2 - \eta L}  \mathbb{E}\big[\|\nabla F (\mathbf{x}_{t_1}) \|^2 | N_{t_1:t}\big] = \frac{ \eta L}{2 - \eta L}  \mathbb{E}\big[\|\nabla F (\mathbf{x}_{t_1}) \|^2\big] 
	\end{align*}
	which proves the first part of Lemma \ref{lemma.l2s}. 

For the second part of Lemma \ref{lemma.l2s}, by calculating the probability of $N_{t_1:t}$ as in \eqref{eq.p_N}, we have 
	\begin{align*}
		 \mathbb{E} \big[ \|\nabla F(\mathbf{x}_t)- \mathbf{v}_t \|^2  \big]  & \stackrel{(c)}{=} \sum_{t_1=0}^{t-1} \mathbb{E}\big[ \|\nabla F(\mathbf{x}_t)- \mathbf{v}_t \|^2  | N_{t_1:t} \big] \mathbb{P}\{ N_{t_1:t} \} \nonumber \\
		 & \leq \sum_{t_1=0}^{t-1} \frac{ \eta L}{2 - \eta L} \mathbb{E}\big[ \|\nabla F(\mathbf{x}_{t_1}) \|^2  \big] \mathbb{P}\{ N_{t_1:t} \} \nonumber \\
		 & = \frac{ \eta L}{2 - \eta L}  \Bigg[ \frac{1}{m} \sum_{\tau=1}^{t-1}  \Big(1-\frac{1}{m}\Big)^{t-\tau} \mathbb{E}\big[ \|\nabla F(\mathbf{x}_{\tau}) \|^2 \big] +  \Big(1-\frac{1}{m}\Big)^t \|\nabla F(\mathbf{x}_0) \|^2 \Bigg]
	\end{align*}
	where (c) uses \eqref{eq.13}, and $\mathbb{E} \big[ \|\nabla F(\mathbf{x}_t)- \mathbf{v}_t \|^2 | N_{t:t} \big] = 0$. The proof is thus completed. 

\subsection{Proof of Theorem \ref{thm.l2s}}
Following Assumption \ref{as.1}, we have
	\begin{align}\label{eq.smoothness}
		F(\mathbf{x}_{t+1}) - F(\mathbf{x}_t) & \leq \big\langle \nabla F(\mathbf{x}_t), \mathbf{x}_{t+1} - \mathbf{x}_t \big\rangle + \frac{L}{2}  \|\mathbf{x}_{t+1} - \mathbf{x}_t \|^2 \nonumber \\
		& = -\eta \big\langle \nabla F(\mathbf{x}_t), \mathbf{v}_t \big\rangle + \frac{\eta^2 L}{2}  \| \mathbf{v}_t \|^2 \nonumber \\
		& = -\frac{\eta}{2}\Big[  \| \nabla F(\mathbf{x}_t) \|^2 + \| \mathbf{v}_t\|^2 - \| \nabla F(\mathbf{x}_t) - \mathbf{v}_t \|^2   \Big]+ \frac{\eta^2 L}{2}  \| \mathbf{v}_t \|^2
	\end{align}
	where the last equation is because $\langle \mathbf{a}, \mathbf{b} \rangle = \frac{1}{2}[\|\mathbf{a} \|^2 + \|\mathbf{b} \|^2 - \|\mathbf{a} - \mathbf{b} \|^2]$. Rearranging the terms, we arrive at
	\begin{align*}
		\|\nabla F(\mathbf{x}_t)\|^2 & \leq \frac{2\big[ F(\mathbf{x}_t) - F(\mathbf{x}_{t+1} )\big]}{\eta} + \| \nabla F(\mathbf{x}_t) - \mathbf{v}_t \|^2 - \big( 1- \eta L\big)  \| \mathbf{v}_t \|^2 \nonumber \\
		& \leq \frac{2\big[ F(\mathbf{x}_t) - F(\mathbf{x}_{t+1} )\big]}{\eta} + \| \nabla F(\mathbf{x}_t) - \mathbf{v}_t \|^2
	\end{align*}
	where the last inequality holds since $\eta < 1/L$. Taking expectation and summing over $t=1,\ldots, T$, we have
	\begin{align*}
		& ~~~~~\sum_{t=1}^T \mathbb{E}  \Big[ \|\nabla F(\mathbf{x}_t)\|^2 \Big]   \leq \frac{2\big[ F(\mathbf{x}_1) - F(\mathbf{x}_{T+1} )\big]}{\eta} + \sum_{t=1}^T \mathbb{E} \Big[ \| \nabla F(\mathbf{x}_t) - \mathbf{v}_t \|^2\Big] \nonumber \\
		& \stackrel{(a)}{\leq} \frac{2\big[ F(\mathbf{x}_1) - F(\mathbf{x}_{T+1} )\big]}{\eta} + \frac{ \eta L}{2 - \eta L} \frac{1}{m}\sum_{t=1}^T \sum_{\tau=1}^{t-1}  \Big(1-\frac{1}{m}\Big)^{t-\tau} \mathbb{E} \big[\|\nabla F(\mathbf{x}_{\tau}) \|^2 \big] \nonumber \\
		& \qquad \qquad  \qquad \qquad \qquad  \qquad  \qquad ~~~~~ + \frac{ \eta L}{2 - \eta L} \sum_{t=1}^T \Big(1-\frac{1}{m}\Big)^t \|\nabla F(\mathbf{x}_0) \|^2 \nonumber \\
		& \stackrel{(b)}{\leq} \frac{2\big[ F(\mathbf{x}_1) - F(\mathbf{x}_{T+1} )\big]}{\eta} +  \frac{ \eta L}{2 - \eta L} \frac{1}{m}\sum_{t=1}^{T-1} \bigg[\sum_{\tau=1}^{T-t}  \Big(1-\frac{1}{m}\Big)^{\tau}\bigg] \mathbb{E} \big[\|\nabla F(\mathbf{x}_{t}) \|^2 \big] \nonumber \\
		& \qquad \qquad  \qquad \qquad \qquad  \qquad  \qquad ~~~~~ +  \frac{ m \eta L}{2 - \eta L} \|\nabla F(\mathbf{x}_0) \|^2 \nonumber \\
		& \stackrel{(c)}{\leq} \frac{2\big[ F(\mathbf{x}_1) - F(\mathbf{x}_{T+1} )\big]}{\eta} + \frac{ \eta L}{2 - \eta L}\sum_{t=1}^T \mathbb{E}\big[\|\nabla F(\mathbf{x}_{t}) \|^2 \big]+  \frac{ m \eta L}{2 - \eta L} \|\nabla F(\mathbf{x}_0) \|^2 
	\end{align*}
	where (a) is the result of Lemma \ref{lemma.l2s}; (b) is by changing the order of summation, and $\sum_{t=1}^T (1-\frac{1}{m})^t \leq m$; and, (c) is again by $\sum_{\tau=1}^{T-t}  (1-\frac{1}{m})^{\tau} \leq m$. Rearranging the terms and dividing both sides by $T$, we have
	\begin{align}\label{eq.thm1.11}
		\bigg(1 - \frac{ \eta L}{2 - \eta L} \bigg) \frac{1}{T}\sum_{t=1}^T \mathbb{E}  \Big[ \|\nabla F(\mathbf{x}_t)\|^2 \Big] & \leq  \frac{2\big[ F(\mathbf{x}_1) - F(\mathbf{x}_{T+1} )\big]}{\eta T} + \frac{ \eta L}{2 - \eta L} \frac{m}{T} \|\nabla F(\mathbf{x}_0) \|^2 \nonumber \\
		& \leq \frac{2\big[ F(\mathbf{x}_1) - F(\mathbf{x}^* )\big]}{\eta T} + \frac{ \eta L}{2 - \eta L} \frac{m}{T} \|\nabla F(\mathbf{x}_0) \|^2.
	\end{align}
		
	Finally, since $\mathbf{v}_0 = \nabla F(\mathbf{x}_0)$, we have 
	\begin{align}\label{eq.x1x0}
		F(\mathbf{x}_1) - F(\mathbf{x}_0) & \leq \big\langle \nabla F(\mathbf{x}_0), \mathbf{x}_{1} - \mathbf{x}_0 \big\rangle + \frac{L}{2}  \|\mathbf{x}_{1} - \mathbf{x}_0 \|^2 \nonumber \\
		& = -\eta \| \nabla F(\mathbf{x}_0) \|^2 + \frac{\eta^2 L}{2} \| \nabla F(\mathbf{x}_0) \|^2 \leq 0
	\end{align}
	where the last inequality follows from $\eta < 1/L$. Hence we have $F(\mathbf{x}_1) \leq  F(\mathbf{x}_0)$, which is applied to \eqref{eq.thm1.11} to have
	\begin{align*}
		\bigg(1 - \frac{ \eta L}{2 - \eta L} \bigg) \frac{1}{T}\sum_{t=1}^T \mathbb{E}  \Big[ \|\nabla F(\mathbf{x}_t)\|^2 \Big] \leq \frac{2\big[ F(\mathbf{x}_0) - F(\mathbf{x}^* )\big]}{\eta T} + \frac{ \eta L}{2 - \eta L} \frac{m}{T} \|\nabla F(\mathbf{x}_0) \|^2.
	\end{align*}
	
	Now if we choose $\eta < 1/L$ such that $1 - \frac{ \eta L}{2 - \eta L} \geq C_\eta $ with $C_\eta$ being a positive constant, then we have
	\begin{align*}
		\mathbb{E}\big[ \| \nabla F(\mathbf{x}_a) \|^2\big] =  \frac{1}{T}\sum_{t=1}^T \mathbb{E}  \Big[ \|\nabla F(\mathbf{x}_t)\|^2 \Big] = {\cal O}\bigg( \frac{ F(\mathbf{x}_0) - F(\mathbf{x}^* )}{\eta TC_\eta}  + \frac{m \eta  L \| \nabla F(\mathbf{x}_0) \|^2 }{TC_\eta} \bigg).	
	\end{align*}


\subsection{Proof of Corollaries \ref{coro.l2s1.1} and \ref{coro.l2s1.2}}
From Theorem \ref{thm.l2s}, it is clear that upon choosing $\eta = {\cal O}(1/L)$, we have $\mathbb{E}\big[ \| \nabla F(\mathbf{x}_a) \|^2\big] = {\cal O}(m/T)$. This means that $T = {\cal O}(m/\epsilon)$ iterations are needed to guarantee $\mathbb{E}\big[ \| \nabla F (\mathbf{x}_a) \|^2\big] = \epsilon$. 

Per iteration requires $\frac{n}{m}+ 2(1-\frac{1}{m})$ IFO calls in expectation. And $n$ IFO calls are required when computing $\mathbf{v}_0$.

Combining these facts together, we have that $\mathbb{E}\big[ \|\nabla F( \mathbf{x}_a) \|^2\big] = {\cal O}(\sqrt{n}/T)$ if $m = \Theta(\sqrt{n})$. And the IFO complexity is $n+ \big[ \frac{n}{m}+ 2(1-\frac{1}{m})\big] T = {\cal O}(n + n/\epsilon)$.

Similarly, if $m = \Theta(n)$, we have $\mathbb{E}\big[ \| \nabla F(\mathbf{x}_a )\|^2\big] = {\cal O}(n/T)$. And the IFO complexity in this case becomes $ {\cal O}(n + n/\epsilon)$.

\subsection{Proof of Corollary \ref{coro.l2s2}}
From Theorem \ref{thm.l2s}, it is clear that with a large $m$, choosing $\eta = {\cal O}(1/\sqrt{m}L)$ leads to $C_\eta \geq 0.5$. Thus we have $\mathbb{E}\big[ \| \nabla F( \mathbf{x}_a) \|^2\big] = {\cal O}(\sqrt{m}/T)$. This translates to the need of $T = {\cal O}(\sqrt{m}/\epsilon)$ iterations to guarantee $\mathbb{E}\big[ \| \nabla F(\mathbf{x}_a) \|^2\big] = \epsilon$.

Choosing $m = \Theta(n)$, we have $\mathbb{E}\big[ \| \nabla F(\mathbf{x}_a) \|^2\big] = {\cal O}(\sqrt{n}/T)$. And the number of IFO calls is $n+ \big[ \frac{n}{m}+ 2(1-\frac{1}{m})\big] T=  {\cal O}(n + \sqrt{n}/\epsilon)$.

\section{Technical Proofs in Section \ref{sec.l2s_nc}}\label{apdx.L2S_nc}	
Using the Bernoulli random variable $B_t$ introduced in \eqref{eq.Bt}, L2S (Alg. \ref{as.2}) can be rewritten in an equivalent form as Alg. \ref{alg.6}.

\begin{algorithm}[H]
    \caption{L2S Equivalent Form}\label{alg.6}
    \begin{algorithmic}[1]
    	\State \textbf{Initialize:} $\mathbf{x}_0 $, $\eta$, $m$, $T$
    	\State Compute $\mathbf{v}_0 = \nabla F(\mathbf{x}_0) $
   		\State $\mathbf{x}_1 = \mathbf{x}_0 - \eta \mathbf{v}_0$
        	\For {$t=1,2,\dots,T$}
			\State Randomly generate $B_t$: $B_t = 1$ w.p. $\frac{1}{m}$, and $B_t = 0$ w.p. $1- \frac{1}{m}$
			\If {$B_t = 1$}, 
			\State $\mathbf{v}_t = \nabla F(\mathbf{x}_t)$ 
			\Else
				\State $\mathbf{v}_t = \nabla f_{i_t}(\mathbf{x}_t) - \nabla f_{i_t}(\mathbf{x}_{t-1}) +  \mathbf{v}_{t-1} $
			\EndIf

			\State $\mathbf{x}_{t+1} = \mathbf{x}_t - \eta \mathbf{v}_t$
		\EndFor
		\State \textbf{Output:} randomly chosen from $\{\mathbf{x}_t\}_{t=1}^T$
	\end{algorithmic}
\end{algorithm}

Recall that a known $N_{t_1:t}$ is equivalent to $B_{t_1}=1, B_{t_1+1}=0, \cdots, B_t=0$. Now we are ready to prove Lemma \ref{lemma.l2s_nc}.


\subsection{Proof of Lemma \ref{lemma.l2s_nc}}

	It can be seen that Lemma \ref{lemma.copy2} still holds for nonconvex problems, thus we have 
	\begin{align}\label{eq.est_error_part1}
		 \mathbb{E} \big[ \|\nabla F(\mathbf{x}_t)- \mathbf{v}_t \|^2 | N_{t_1:t} \big] & \leq \sum_{\tau=t_1+1}^{t}\mathbb{E} \big[ \| \mathbf{v}_\tau  - \mathbf{v}_{\tau-1} \|^2 | N_{t_1:t}	\big] \nonumber \\
		 & = \sum_{\tau=t_1+1}^{t}\mathbb{E} \big[ \| \nabla f_{i_{\tau}}(\mathbf{x}_\tau) -  \nabla f_{i_{\tau}}(\mathbf{x}_{\tau-1})\|^2 | N_{t_1:t}	\big] \nonumber \\
		 & \leq \eta^2 L^2 \sum_{\tau=t_1+1}^{t} \mathbb{E} \big[ \| \mathbf{v}_{\tau-1}  \|^2 | N_{t_1:t}	\big] = \eta^2 L^2 \sum_{\tau=t_1}^{t-1} \mathbb{E} \big[ \| \mathbf{v}_{\tau}  \|^2 | N_{t_1:t}	\big] 
	\end{align}
	where the last inequality follows from Assumption \ref{as.1} and $\mathbf{x}_{\tau} = \mathbf{x}_{\tau-1} - \eta \mathbf{v}_{\tau-1}$. The first part of this lemma is thus proved. Next, we have
	\begin{align*}
	&~~~~~~\mathbb{E} \big[ \|\nabla F(\mathbf{x}_t)- \mathbf{v}_t \|^2 \big] \stackrel{(a)}{=} \sum_{t_1 = 0}^{t-1}  \mathbb{E} \big[ \|\nabla F(\mathbf{x}_t)- \mathbf{v}_t \|^2 | N_{t_1:t} \big] \mathbb{P} \big\{ N_{t_1:t} \big\} \nonumber \\
		&  \stackrel{(b)}{\leq} \eta^2 L^2  \sum_{t_1 = 0}^{t-1}  \sum_{\tau=t_1}^{t-1} \mathbb{E} \big[ \| \mathbf{v}_{\tau}  \|^2 | N_{t_1:t}	\big]  \mathbb{P} \big\{ N_{t_1:t} \big\}  \stackrel{(c)}{=} \eta^2 L^2  \sum_{\tau=0}^{t-1} \bigg[ \sum_{t_1 = 0}^{\tau}   \mathbb{E} \big[ \| \mathbf{v}_{\tau}  \|^2 | N_{t_1:t}	\big]  \mathbb{P} \big\{ N_{t_1:t} \big\}  \bigg] \nonumber \\
		& \stackrel{(d)}{=} \eta^2 L^2  \sum_{\tau=0}^{t-1} \bigg[  \mathbb{E} \big[ \| \mathbf{v}_{\tau}  \|^2 \big] - \sum_{t_1 = \tau+1}^{t}   \mathbb{E} \big[ \| \mathbf{v}_{\tau}  \|^2 | N_{t_1:t}	\big]  \mathbb{P} \big\{ N_{t_1:t} \big\}  \bigg] \nonumber \\
		& \stackrel{(e)}{=} \eta^2 L^2  \sum_{\tau=0}^{t-1} \bigg[  \mathbb{E} \big[ \| \mathbf{v}_{\tau}  \|^2 \big] - \sum_{t_1 = \tau+1}^{t}   \mathbb{E} \big[ \| \mathbf{v}_{\tau}  \|^2 	\big]  \mathbb{P} \big\{ N_{t_1:t} \big\}  \bigg] \nonumber \\
		& = \eta^2 L^2  \sum_{\tau=0}^{t-1} \bigg[ \sum_{t_1 = 0}^{\tau}    \mathbb{P} \big\{ N_{t_1:t} \big\}  \bigg] \mathbb{E} \big[ \| \mathbf{v}_{\tau}  \|^2 	\big]  = \eta^2 L^2  \sum_{\tau=0}^{t-1} \bigg( 1 - \frac{1}{m} \bigg)^{t - \tau} \mathbb{E} \big[ \| \mathbf{v}_{\tau}  \|^2 	\big]
	\end{align*}
	where (a) is by Lemma \ref{lemma.p_N} (or law of total probability) and $\mathbb{E} [ \|\nabla F(\mathbf{x}_t)- \mathbf{v}_t \|^2 | N_{t:t} ] = 0$; (b) is obtained by plugging \eqref{eq.est_error_part1} in; (c) is established by changing the order of summation; (d) is again by Lemma \ref{lemma.p_N} (or law of total probability); and (e) is because of the independence of $\mathbf{v}_{\tau}$ and $N_{t_1:t}$ when $t_1 > \tau$, that is, $\mathbb{E} [ \| \mathbf{v}_{\tau}  \|^2 | N_{t_1:t}] = \mathbb{E} [ \| \mathbf{v}_{\tau}  \|^2 | B_{t_1}=1, B_{t_1+1}=0, \ldots, B_t=0] = \mathbb{E} [ \| \mathbf{v}_{\tau}  \|^2]$. To be more precise, given $t_1 > \tau$, the randomness of $\mathbf{v}_\tau$ comes from $B_1, B_2, \ldots B_\tau$ and $i_1, i_2, \cdots,i_\tau$, thus is independent with $B_{t_1}, B_{t_1+1}, \ldots, B_{t}$.

\subsection{Proof of Theorem \ref{thm.l2s_nc}}
	Following the same steps of \eqref{eq.smoothness} in Theorem \ref{thm.l2s}, we have
	\begin{align*}
		\|\nabla F(\mathbf{x}_t)\|^2 & \leq \frac{2\big[ F(\mathbf{x}_t) - F(\mathbf{x}_{t+1} )\big]}{\eta} + \| \nabla F(\mathbf{x}_t) - \mathbf{v}_t \|^2 - \big( 1- \eta L\big)  \| \mathbf{v}_t \|^2.
	\end{align*}
	Taking expectation and summing over $t$, we have
	\begin{align}\label{eq.nc_1}
		& ~~~~~\sum_{t=1}^T \mathbb{E}  \Big[ \|\nabla F(\mathbf{x}_t)\|^2 \Big]   \leq \frac{2\big[ F(\mathbf{x}_1) - F(\mathbf{x}^* )\big]}{\eta} + \sum_{t=1}^T \mathbb{E} \Big[ \| \nabla F(\mathbf{x}_t) - \mathbf{v}_t \|^2\Big] - \big( 1- \eta L\big) \sum_{t=1}^T \mathbb{E}\big[ \| \mathbf{v}_t \|^2\big]  \nonumber \\
		& \stackrel{(a)}{\leq} \frac{2\big[ F(\mathbf{x}_1) - F(\mathbf{x}^* )\big]}{\eta} + \eta^2 L^2  \sum_{t=1}^T \sum_{\tau=0}^{t-1} \bigg( 1 - \frac{1}{m} \bigg)^{t - \tau} \mathbb{E} \big[ \| \mathbf{v}_{\tau}  \|^2 	\big] - \big( 1- \eta L\big) \sum_{t=1}^T \mathbb{E}\big[ \| \mathbf{v}_t \|^2\big] \nonumber \\
		& \stackrel{(b)}{\leq} \frac{2\big[ F(\mathbf{x}_1) - F(\mathbf{x}^* )\big]}{\eta} + \eta^2 L^2  \sum_{t=1}^T \sum_{\tau=0}^{t-1} \bigg( 1 - \frac{1}{m} \bigg)^{t - \tau} \mathbb{E} \big[ \| \mathbf{v}_{\tau}  \|^2 	\big] - \big( 1- \eta L\big) \sum_{t=1}^{T-1} \mathbb{E}\big[ \| \mathbf{v}_t \|^2\big] \nonumber \\
		& \stackrel{(c)}{\leq} \frac{2\big[ F(\mathbf{x}_1) - F(\mathbf{x}^* )\big]}{\eta} + m \eta^2 L^2  \sum_{t=0}^{T-1} \mathbb{E} \big[ \| \mathbf{v}_{t}  \|^2 	\big] - \big( 1- \eta L\big) \sum_{t=1}^{T-1} \mathbb{E}\big[ \| \mathbf{v}_t \|^2\big] \nonumber \\
		& = \frac{2\big[ F(\mathbf{x}_1) - F(\mathbf{x}^* )\big]}{\eta} + m \eta^2 L^2  \| \mathbf{v}_0 \|^2 + \big(m \eta^2 L^2 + \eta L -1 \big) \sum_{t=1}^{T-1} \mathbb{E} \big[ \| \mathbf{v}_t  \|^2 	\big]
	\end{align}
	where (a) is by Lemma \ref{lemma.l2s_nc}; 
	(b) holds when $1-\eta L \geq 0$; 
and (c) is by exchanging the order of summation and $\sum_{t=1}^{T-1} (1-\frac{1}{m})^t \leq m$. Upon choosing $\eta$ such that $m \eta^2 L^2 + \eta L -1 \leq 0$, i.e., $\eta \in (0,  \frac{\sqrt{4m+1} - 1}{2mL}] = {\cal O}\big( \frac{1}{L \sqrt{m}} \big)$, we can eliminate the last term in \eqref{eq.nc_1}. Plugging $m$ in and dividing both sides by $T$, we arrive at
	\begin{align*}
		 \mathbb{E}  \Big[ \|\nabla F(\mathbf{x}_a)\|^2 \Big] & = \frac{1}{T}\sum_{t=1}^T \mathbb{E}  \Big[ \|\nabla F(\mathbf{x}_t)\|^2 \Big]   \leq  \frac{2\big[ F(\mathbf{x}_1) - F(\mathbf{x}^* )\big]}{\eta T} + \frac{m \eta^2 L^2}{T}  \| \nabla F(\mathbf{x}_0)  \|^2 \nonumber \\
		& \stackrel{(d)}{\leq} \frac{2\big[ F(\mathbf{x}_0) - F(\mathbf{x}^* )\big]}{\eta T} + \frac{m \eta^2 L^2}{T}  \| \nabla F(\mathbf{x}_0)  \|^2 \nonumber \\
		& = {\cal O} \bigg( \frac{ L \sqrt{m} \big[ F(\mathbf{x}_0) - F(\mathbf{x}^* )\big]}{ T} + \frac{ \| \nabla F(\mathbf{x}_0)  \|^2 }{T} \bigg)
	\end{align*}
	where (d) is because $F(\mathbf{x}_0)\geq F(\mathbf{x}_1)$ when $\eta \leq 2/L$, which we have already seen from \eqref{eq.x1x0}. The proof is thus completed.


\subsection{Proof of Corollary \ref{coro.l2s_nc}}

From Theorem \ref{thm.l2s_nc}, choosing $\eta = {\cal O}(1/L\sqrt{m})$, we have $\mathbb{E}\big[ \| \nabla F(\mathbf{x}_a) \|^2\big] = {\cal O}(\sqrt{m}/T)$. This means that $T = {\cal O}(\sqrt{m}/\epsilon)$ iterations are required to ensure $\mathbb{E}\big[ \| \nabla F (\mathbf{x}_a) \|^2\big] = \epsilon$.

Per iteration it takes in expectation $\frac{n}{m}+ 2(1-\frac{1}{m})$ IFO calls. And $n$ IFO calls are required for computing $\mathbf{v}_0$

Hence choosing $m = \Theta(n)$, the IFO complexity is $n+ \big[ \frac{n}{m}+ 2(1-\frac{1}{m})\big] T = {\cal O}(n + \sqrt{n}/\epsilon)$.

\section{Technical Proofs in Section \ref{sec.l2s-sc}}

\subsection{Proof of Lemma \ref{lemma.sarah_sc}}
We borrow the following lemmas from \citep{nguyen2017} and summarize them below. 
\begin{lemma}\label{lemma.sc_conv}
	\cite[Theorem 1a]{nguyen2017} Suppose that Assumptions \ref{as.1} - \ref{as.3} hold. Choosing step size $\eta \leq 2/L$ in SARAH (Alg. \ref{alg.1}), then for a particular inner loop $s$ and any $t\geq 1$, we have
	\begin{align*}
		\mathbb{E}\big[ \|\mathbf{v}_t^s \|^2 \big] \leq \bigg[1- \bigg(\frac{2}{\eta L} -1 \bigg)\mu^2 \eta^2 \bigg]^t \mathbb{E}\big[ \|\nabla F(\tilde{\mathbf{x}}^{s-1}) \|^2 \big].
	\end{align*}	
\end{lemma}

\begin{lemma}\label{lemma.sc_conv2}
	\cite[Theorem 1b]{nguyen2017} Suppose that Assumptions \ref{as.1} and \ref{as.4} hold. Choosing step size $\eta < 2/(\mu + L)$ in SARAH (Alg. \ref{alg.1}), then for a particular inner loop $s$ and any $t\geq 1$, we have
	\begin{align*}
		\mathbb{E}\big[ \|\mathbf{v}_t^s \|^2 \big] \leq \bigg[1- \frac{2\mu L\eta}{\mu+ L} \bigg]^t \mathbb{E}\big[ \|\nabla F(\tilde{\mathbf{x}}^{s-1}) \|^2 \big].
	\end{align*}	
\end{lemma}

Now we are ready to prove Lemma \ref{lemma.sarah_sc}.

\textbf{Case 1: Assumptions \ref{as.1} -- \ref{as.3} hold}.
Following Assumption \ref{as.1}, we have 	\begin{align}\label{eq.mplus1_useful}
		F(\mathbf{x}_{t+1}^s) - F(\mathbf{x}_t^s) & \leq  -\frac{\eta}{2}\Big[  \| \nabla F(\mathbf{x}_t^s) \|^2 + \| \mathbf{v}_t^s\|^2 - \| \nabla F(\mathbf{x}_t^s) - \mathbf{v}_t^s \|^2   \Big]+ \frac{(\eta)^2 L}{2}  \| \mathbf{v}_t^s \|^2.
	\end{align}
The derivation is exactly the same as \eqref{eq.smoothness}, so we do not repeat it here. Rearranging the terms and dividing both sides with $\eta/2$, we have
	\begin{align*}
		\|\nabla F(\mathbf{x}_t^s)\|^2 & \leq \frac{2\big[ F(\mathbf{x}_t^s) - F(\mathbf{x}_{t+1}^s )\big]}{\eta} + \| \nabla F(\mathbf{x}_t^s ) - \mathbf{v}_t^s \|^2 - \big( 1- \eta L\big)  \| \mathbf{v}_t^s \|^2 \nonumber \\
		& \stackrel{(a)}{\leq} \frac{2\big\langle \nabla F(\mathbf{x}_t^s), \mathbf{x}_t^s - \mathbf{x}_{t+1}^s \big\rangle}{\eta} + \| \nabla F(\mathbf{x}_t^s) - \mathbf{v}_t^s \|^2 - \big( 1- \eta L\big)  \| \mathbf{v}_t^s \|^2\nonumber \\
		& \stackrel{(b)}{\leq} \frac{2}{\eta} \bigg[  \frac{\delta \|\nabla F(\mathbf{x}_t^s) \|^2}{2} + \frac{\| \mathbf{x}_t^s - \mathbf{x}_{t+1}^s\|^2 }{2\delta}\bigg] + \| \nabla F(\mathbf{x}_t^s) - \mathbf{v}_t^s \|^2 - \big( 1- \eta L\big)  \| \mathbf{v}_t^s \|^2\nonumber \\
	\end{align*}
	where (a) follows from the convexity of $F$; (b) uses Young's inequality with $\delta>0$ to be specified later. Since $\mathbf{x}_{t+1}^s = \mathbf{x}_t^s - \eta \mathbf{v}_t^s$, rearranging the terms we have
	\begin{align*}
		\bigg(1 - \frac{\delta}{\eta}\bigg)\|\nabla F(\mathbf{x}_t^s)\|^2 
		& \leq  \| \nabla F(\mathbf{x}_t^s) - \mathbf{v}_t^s \|^2 - \bigg( 1- \eta L - \frac{\eta}{\delta}\bigg)  \| \mathbf{v}_t^s \|^2.
	\end{align*}
	Choosing $\delta = 0.5 \eta$, we have
	\begin{align}\label{eq.conv1111}
		\frac{1}{2}\|\nabla F(\mathbf{x}_t^s)\|^2 
		& \leq  \| \nabla F(\mathbf{x}_t^s) - \mathbf{v}_t^s \|^2 + \big( 1 + \eta L \big)  \| \mathbf{v}_t^s \|^2.
	\end{align}
	Then, taking expectation w.r.t. ${\cal F}_{t-1}$, applying Lemma \ref{lemma.momt} to $\mathbb{E}[\| \nabla F(\mathbf{x}_t^s) - \mathbf{v}_t^s \|^2]$ and Lemma \ref{lemma.sc_conv} to  $\mathbb{E}[\| \mathbf{v}_t^s \|^2]$, with $t = m$ we have
	\begin{align*}
		\frac{1}{2}\mathbb{E} \big[\|\nabla F(\mathbf{x}_m^s)\|^2 \big]
		& \leq  \frac{\eta L}{2- \eta L} \|\nabla F(\tilde{\mathbf{x}}^{s-1}) \|^2  + \big( 1 + \eta L \big) \bigg[1- \bigg(\frac{2}{\eta L} -1 \bigg)\mu^2 \eta^2 \bigg]^m \mathbb{E}\big[ \|\nabla F(\tilde{\mathbf{x}}^{s-1}) \|^2 \big].
	\end{align*}
Multiplying both sides by $2$ completes the proof.

\textbf{Case 2: Assumptions \ref{as.1} and \ref{as.4} hold}. Using exactly same arguments as Case 1 we can arrive at \eqref{eq.conv1111}. Now applying Lemma \ref{lemma.sc_conv2}, we have 
	\begin{align*}
		\frac{1}{2}\mathbb{E} \big[\|\nabla F(\mathbf{x}_m^s)\|^2 \big]
		& \leq  \frac{\eta L}{2- \eta L} \|\nabla F(\tilde{\mathbf{x}}^{s-1}) \|^2  + \big( 1 + \eta L \big) \Big(1- \frac{2\mu L\eta}{\mu+ L}  \Big)^m \mathbb{E}\big[ \|\nabla F(\tilde{\mathbf{x}}^{s-1}) \|^2 \big] \nonumber \\
		& =  \frac{\eta L}{2- \eta L} \|\nabla F(\tilde{\mathbf{x}}^{s-1}) \|^2  + \big( 1 + \eta L \big) \Big(1- \frac{2 L\eta}{1+ \kappa}  \Big)^m \mathbb{E}\big[ \|\nabla F(\tilde{\mathbf{x}}^{s-1}) \|^2 \big].
	\end{align*}
Multiplying both sides by $2$ completes the proof.


\subsection{Proof of Theorem \ref{thm.l2s_sc}}
We will only analyze case 1 where Assumptions \ref{as.1} -- \ref{as.3} hold. The other case where Assumptions \ref{as.1} and \ref{as.4} are true can be analyzed in the same manner.

For analysis, let sequence $\{0, t_1, t_2, \ldots, t_N\}$, be the iteration indices where $B_{t_i} = 1$ ($0$ is automatically contained since at the beginning of L2S-SC, $\mathbf{v}_0$ is calculated). 
For a given sequence $\{0, t_1, t_2, \ldots, t_S\}$, it can be seen that due to the step back in Line 7 of Alg. \ref{alg.l2s_sc}, $\mathbf{x}_{t_i - 1}$ plays the role of the starting point of an inner loop of SARAH; while $\mathbf{x}_{t_{i+1} - 1}$ is analogous to $\mathbf{x}_{m}^s$ of SARAH's inner loop. Define $\mathbf{x}_{-1} = \mathbf{x}_0$ and
	\begin{align}\label{eq.unknown}
		\lambda_{i+1} := \bigg\{ \frac{2 \eta L}{2- \eta L} + \big( 2 + 2\eta L \big) \bigg[1- \bigg(\frac{2}{\eta L} -1 \bigg)\mu^2 \eta^2 \bigg]^{t_{i+1} - t_i} \bigg\}. 	
	\end{align}
	Using similar arguments of Lemma \ref{lemma.sarah_sc}, when $\eta \leq 2/(3L)$, it is guaranteed to have 
	\begin{align}\label{eq.l2s_sc_bound1}
		\mathbb{E} \big[\|\nabla F(\mathbf{x}_{t_S - 1}) \|^2 \big| \{0, t_1, t_2, \ldots, t_S\} \big] & \leq \lambda_{S} \mathbb{E}\big[ \|\nabla F(\mathbf{x}_{t_{S-1}}) \|^2 \big| \{0, t_1, t_2, \ldots, t_S\} \big] \nonumber \\
		& = \lambda_{S} \mathbb{E}\big[ \|\nabla F(\mathbf{x}_{t_{S-1}-1}) \|^2 \big| \{0, t_1, t_2, \ldots, t_S\} \big] \nonumber \\
		& \leq \lambda_{S} \lambda_{S-1} \ldots \lambda_1  \|\nabla F(\mathbf{x}_0) \|^2.
	\end{align}

For convenience, let us define 
\begin{align*}
	\theta:= 1- \bigg(\frac{2}{\eta L} -1 \bigg)\mu^2 \eta^2. 
\end{align*}
Note that choosing $\eta$ properly we can have $\theta < 1$. Now it can be seen that 
\begin{align*}
	\mathbb{E}[  \theta^{t_{i+1}- t_i}|t_i] & \leq \sum_{j=1}^\infty \frac{1}{m} \bigg( 1 -\frac{1}{m}\bigg)^{j-1} \theta^j \leq \frac{1}{m-1} \frac{\theta(1- \frac{1}{m})}{1 - \theta(1-\frac{1}{m})}.
\end{align*}
Note that this inequality is irrelevant with $t_i$. Thus if we further take expectation w.r.t. $t_i$, we arrive at
\begin{align}\label{eq.explambda}
	\mathbb{E}[  \theta^{t_{i+1}- t_i}] \leq \frac{1}{m-1} \frac{\theta(1- \frac{1}{m})}{1 - \theta(1-\frac{1}{m})}.
\end{align}
Plugging \eqref{eq.explambda} into \eqref{eq.unknown} we have
\begin{align*}
	\mathbb{E}[\lambda_i] \leq  \frac{2 \eta L}{2- \eta L} + \frac{ 2 + 2\eta L}{m-1}  	\frac{\theta(1- \frac{1}{m})}{1 - \theta(1-\frac{1}{m})} := \lambda, \forall i.
\end{align*}
Note that the randomness of $\lambda_{i+1}$ comes from $t_{i+1} - t_i$, which is the length of the interval between the calculation of two snapshot gradient. Since $\mathbb{P}\{t_{i+1} - t_i= u, t_{i+2} - t_{i+1} = v\} = \mathbb{P}\{t_{i+1} - t_i= u\} \mathbb{P}\{ t_{i+2} - t_{i+1} = v\} $ for positive integers $u$ and $v$, it can be seen $\{t_{i+1} - t_i\}$ are mutually independent, which further leads to the mutual independence of $\lambda_1, \lambda_2, \ldots, \lambda_S$. Therefore, taking expectation w.r.t. $\{0, t_1, t_2, \ldots, t_S\}$ on both sides of \eqref{eq.l2s_sc_bound1}, we have
\begin{align*}
	\mathbb{E} \big[\|\nabla F(\mathbf{x}_{t_S - 1}) \|^2 \big] 
	& =\mathbb{E} [\lambda_{S} \lambda_{S-1} \ldots \lambda_1]  \|\nabla F(\mathbf{x}_0) \|^2 \leq \lambda^S  \|\nabla F(\mathbf{x}_0) \|^2 
\end{align*}
which completes the proof.

\subsection{When to Use An $n$-dependent Step Size in Convex Problems}\label{apdx.when}

\begin{figure}[H]
	\centering
	\begin{tabular}{cccc}
		\hspace{-0.2cm}
		\includegraphics[width=.23\textwidth]{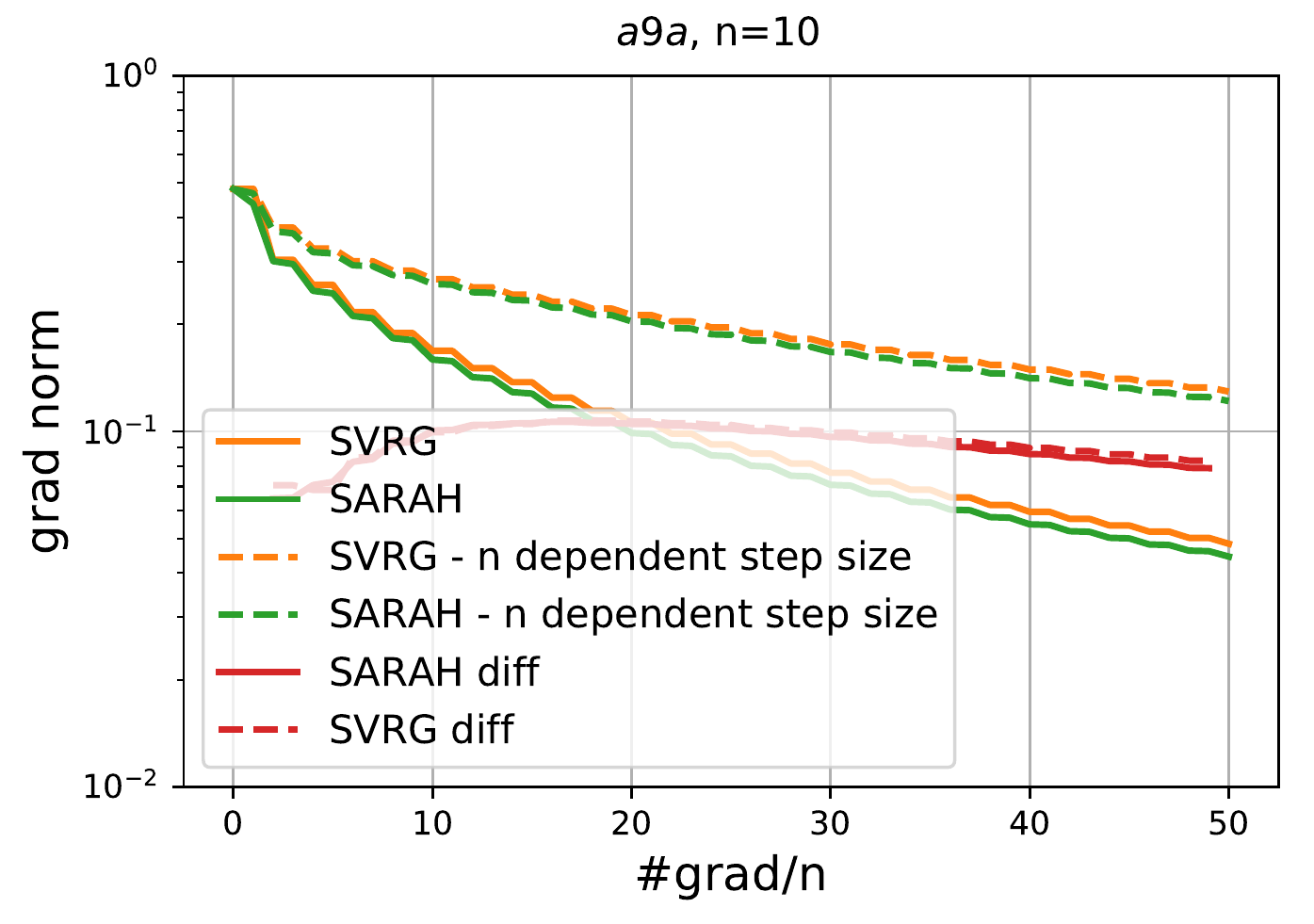}&
		\hspace{-0.2cm}
		\includegraphics[width=.23\textwidth]{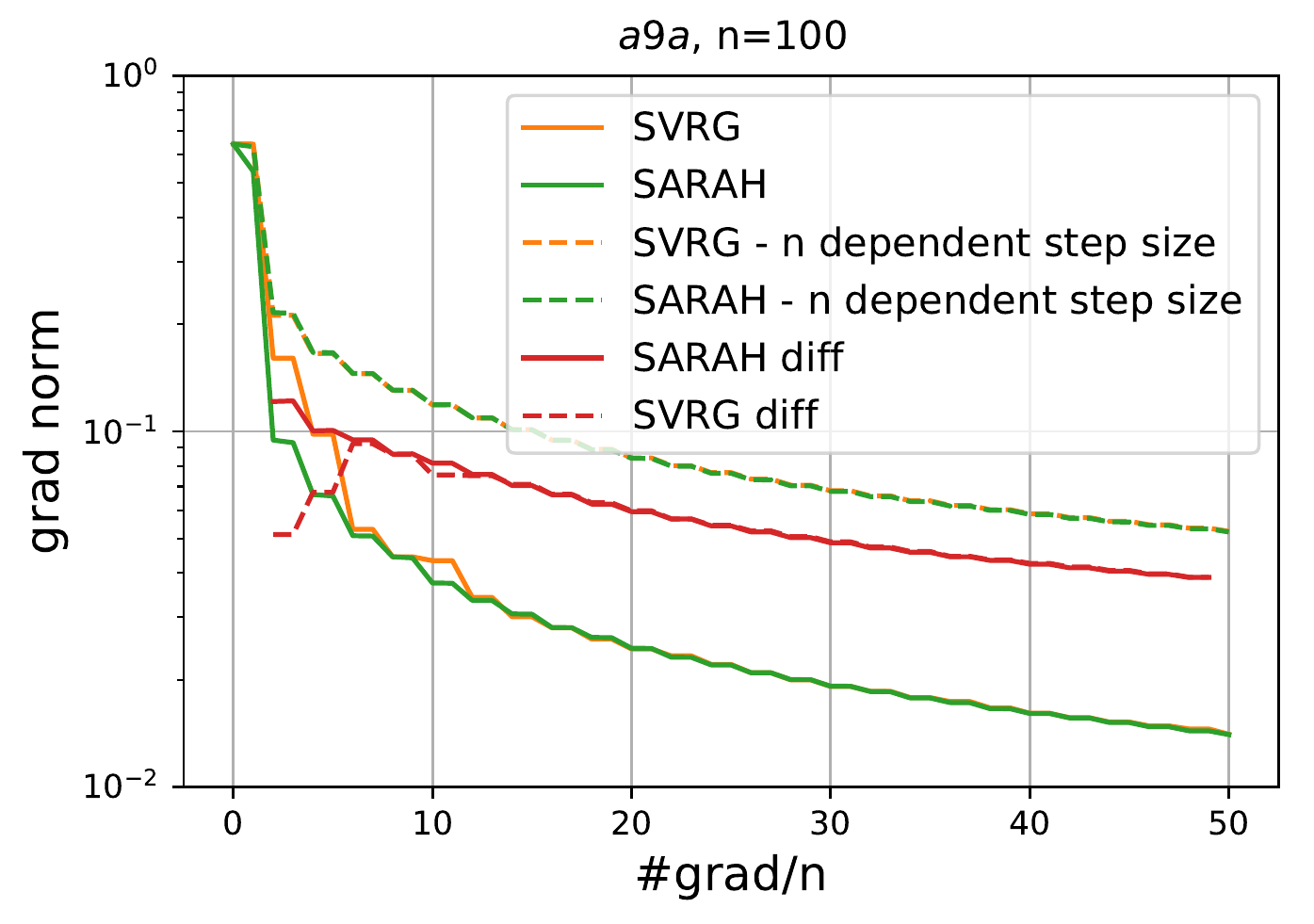}&
		\hspace{-0.2cm}
		\includegraphics[width=.23\textwidth]{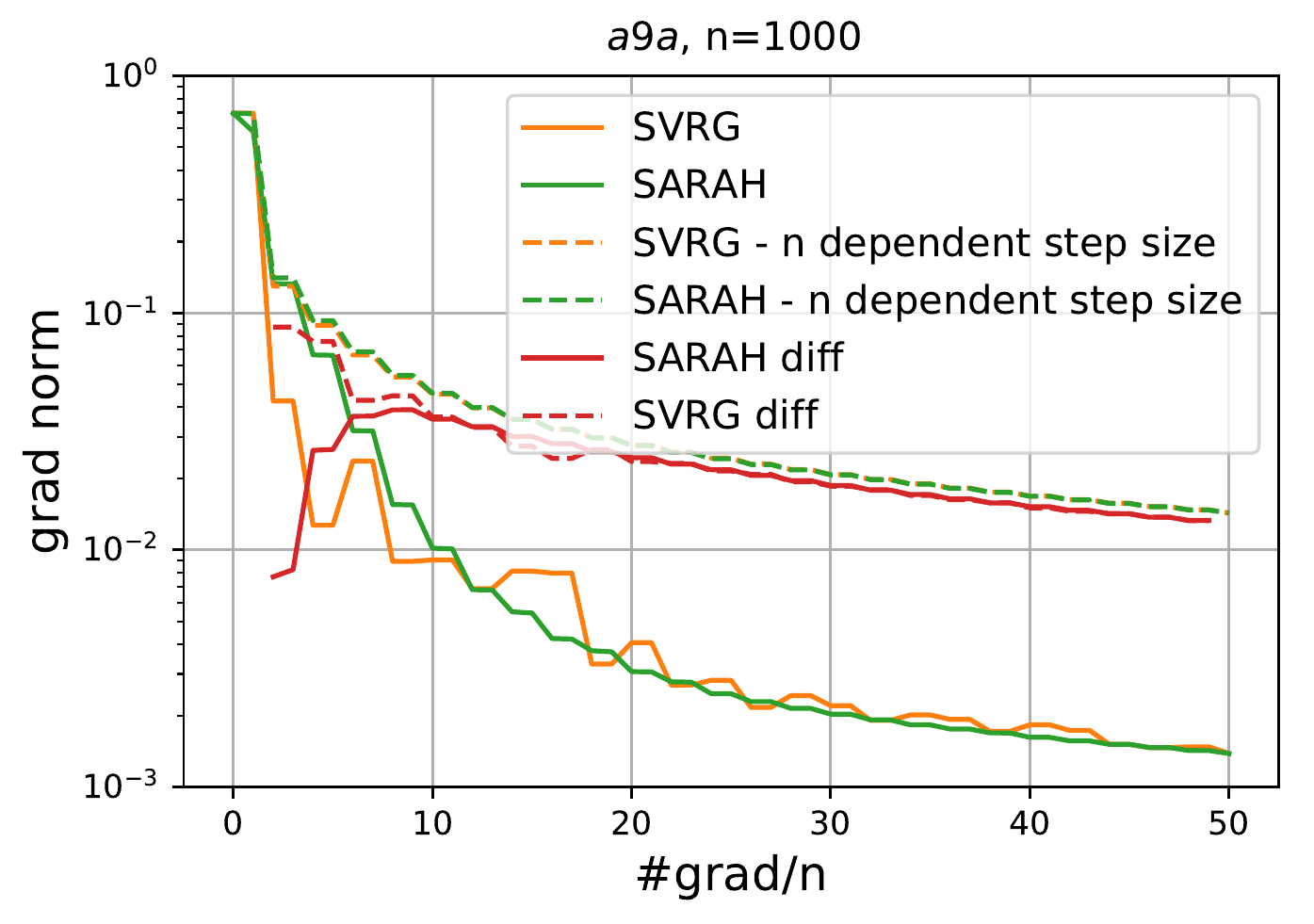}&
		\hspace{-0.2cm}
		\includegraphics[width=.23\textwidth]{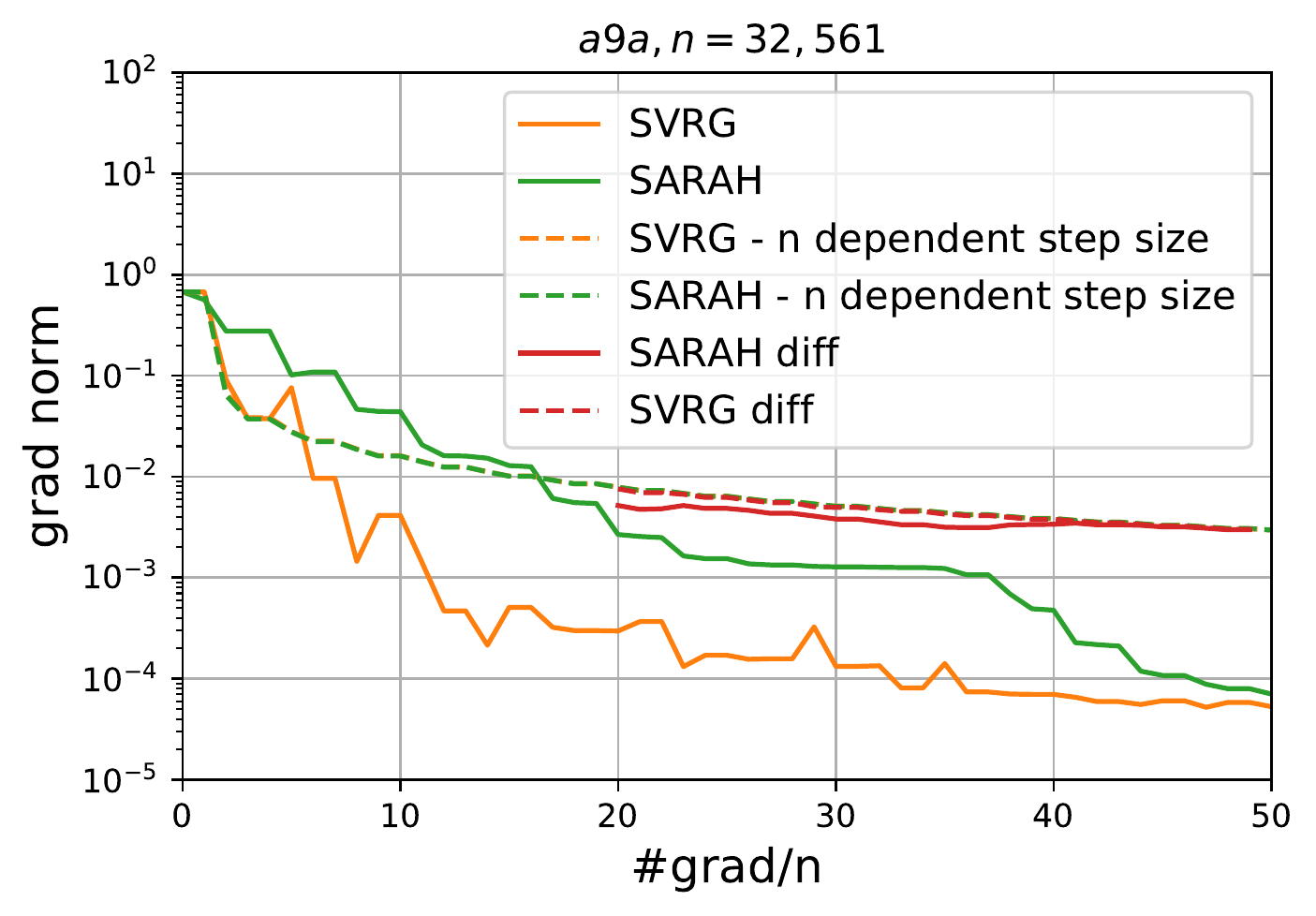}
	\\
		\hspace{-0.2cm}
		\includegraphics[width=.23\textwidth]{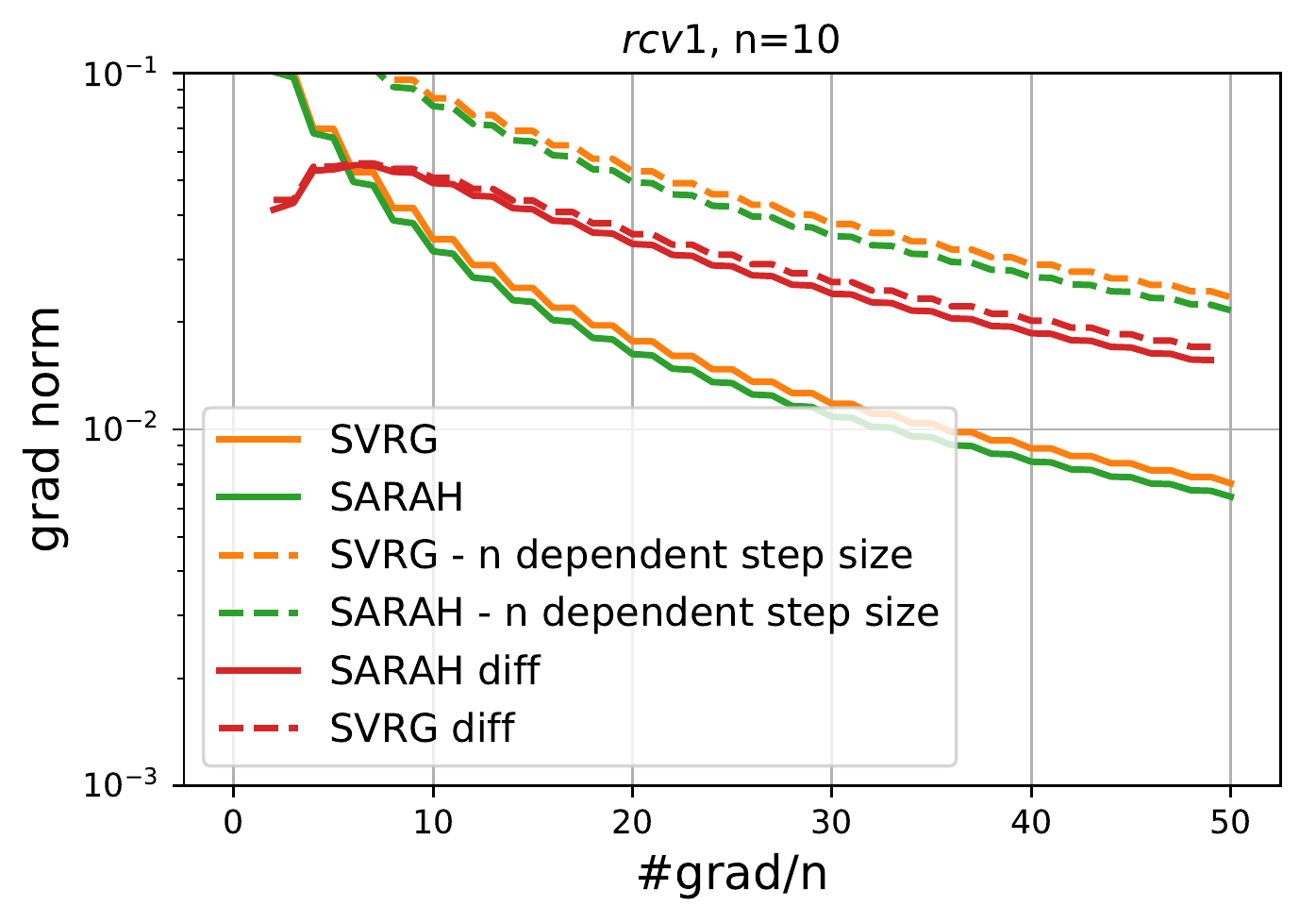}&
		\hspace{-0.2cm}
		\includegraphics[width=.23\textwidth]{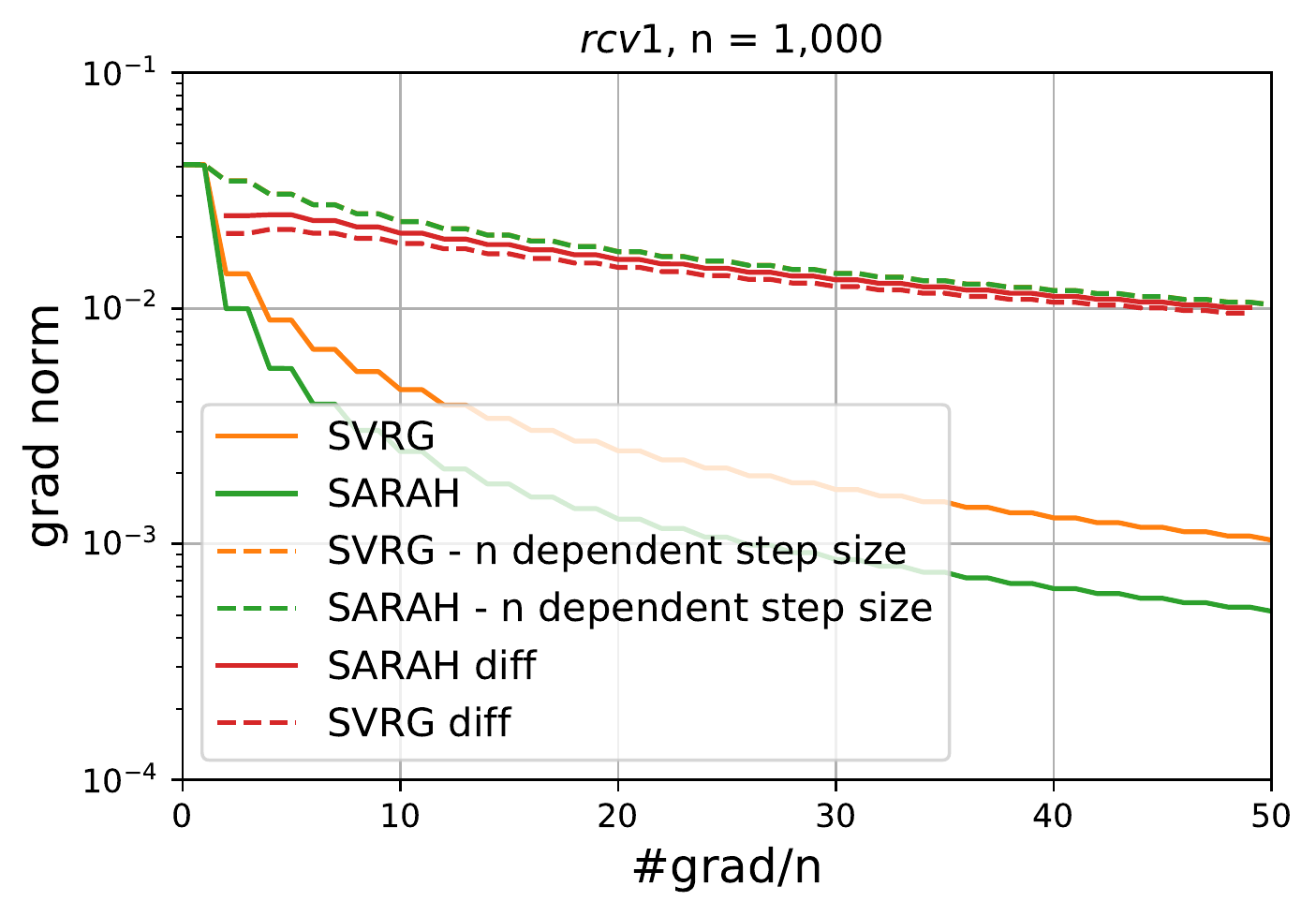}&
		\hspace{-0.2cm}
		\includegraphics[width=.23\textwidth]{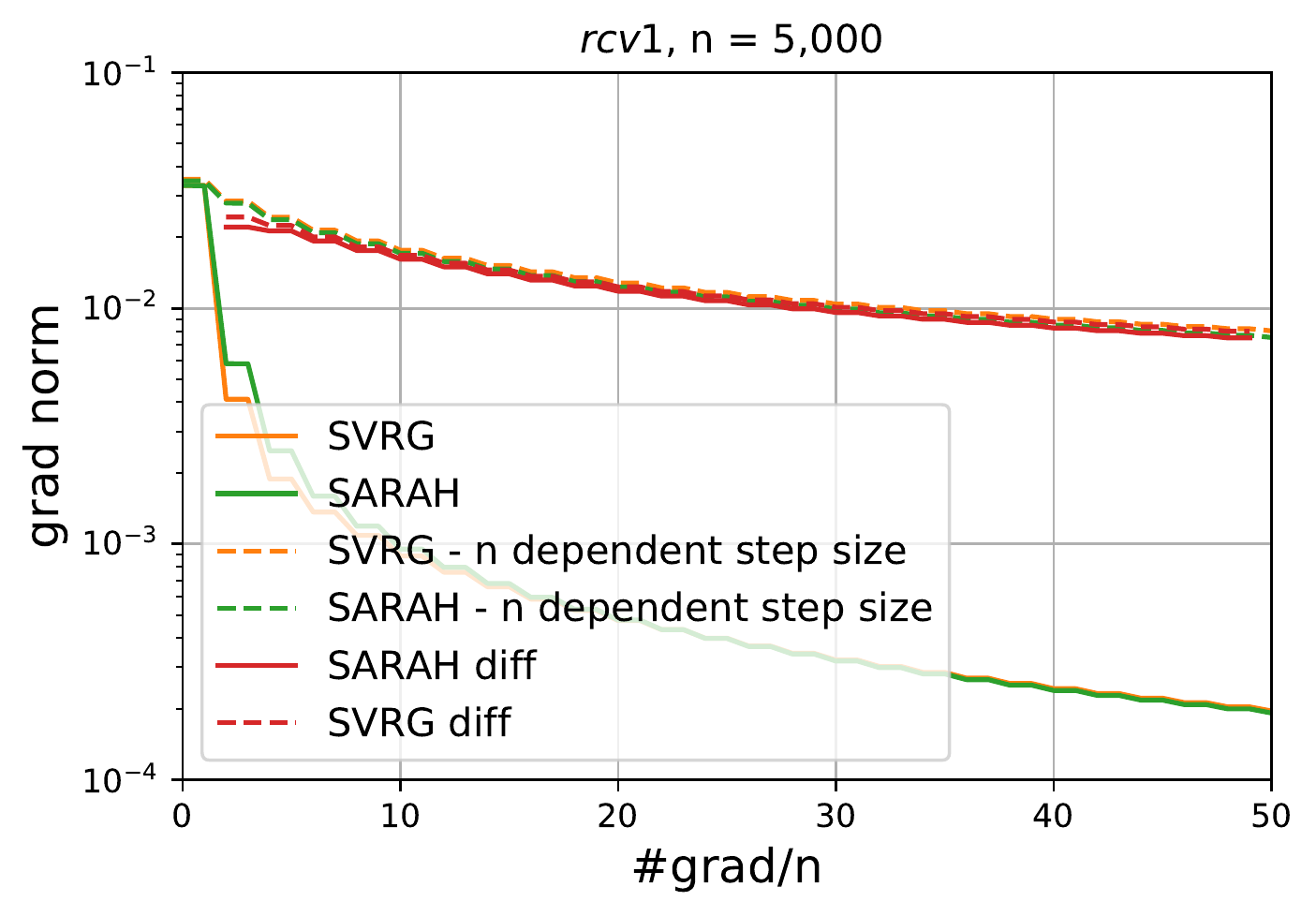}&
		\hspace{-0.2cm}
		\includegraphics[width=.23\textwidth]{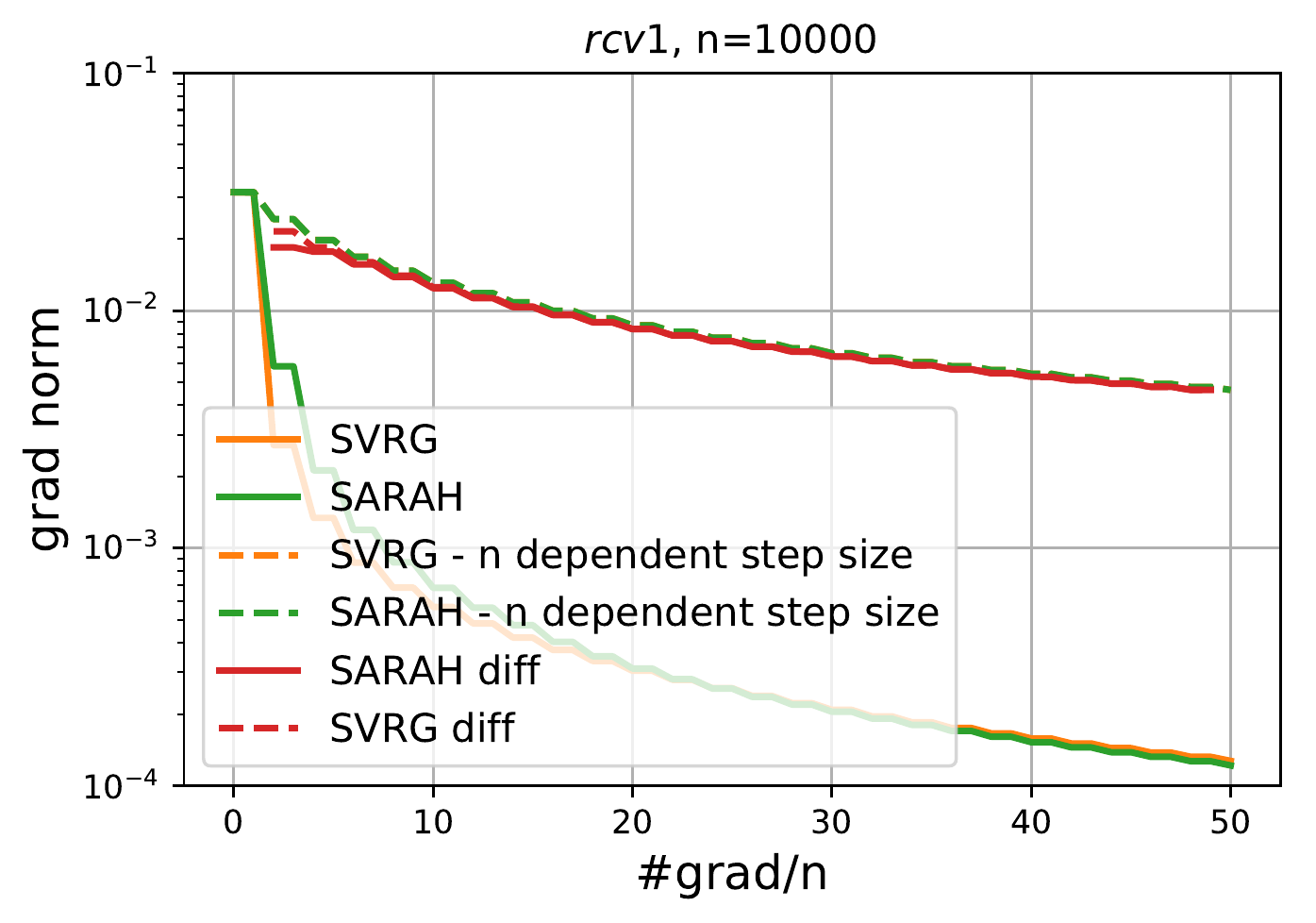}
	\end{tabular}
	\caption{Performances of $n$-dependent step size and $n$-independent step size under on subsample datasets \textit{rcv1} and \textit{a9a}.}
	 \label{fig.111}
\end{figure}

We perform SVRG and SARAH with $n$-dependent/independent step sizes to solve logistic regression problems on subsampled \textit{rcv1} and \textit{a9a}. The results can be found in Fig. \ref{fig.111}. It can be seen that $n$-independent step sizes perform better than those of $n$-dependent step sizes in all the tests. In addition, as $n$ increases, i) the gradient norm of solutions obtained via $n$-dependent step sizes becomes smaller; and ii) the performance gap between $n$-dependent and $n$-independent step sizes reduces. These observations suggest $n$-dependent step sizes can reveal their merits when $n$ is extremely large (at least it should be larger than the size of \textit{a9a}, which is $n = 32561$).

\section{Boosting the Practical Merits of SARAH}\label{sec.d2s}

\begin{algorithm}[H]
    \caption{D2S}\label{alg.3}
    \begin{algorithmic}[1]
    	\State \textbf{Initialize:} $\tilde{\mathbf{x}}_0 $, $\eta$, $m$, $S$
    	\For {$s=1,2,\dots,S$}
			\State $\mathbf{x}_0^s = \tilde{\mathbf{x}}^{s-1}$
			\State $\mathbf{v}_0^s =  \nabla F (\mathbf{x}_0^s )$
			\State $\mathbf{x}_1^s = \mathbf{x}_0^s - \eta \mathbf{v}_0^s$
			\For {$t=1,2,\dots,m$}
				\State Sample $i_t$ according to $\mathbf{p}_t^s$ in \eqref{eq.opt_p_sol2}
				\State Compute $\mathbf{v}_t^s$ via \eqref{eq.d2s_est}
				\State $\mathbf{x}_{t+1}^s = \mathbf{x}_t^s - \eta \mathbf{v}_t^s$
			\EndFor
			\State $\tilde{\mathbf{x}}^s$ uniformly rnd. chosen from $\{\mathbf{x}_t^s\}_{t=0}^m$
		\EndFor
		\State \textbf{Output:} $\tilde{\mathbf{x}}^S$
		\end{algorithmic}
\end{algorithm}

\begin{assumption}\label{as.11}
Each $f_i: \mathbb{R}^d \rightarrow \mathbb{R}$ has $L_i$-Lipchitz gradient, and $F$ has $L_F$-Lipchitz gradient; that is, $\|\nabla f_i(\mathbf{x}) - \nabla f_i(\mathbf{y}) \| \leq L_i \| \mathbf{x}-\mathbf{y} \|$, and $\|\nabla F(\mathbf{x}) - \nabla F(\mathbf{y}) \| \leq L_F \| \mathbf{x}-\mathbf{y} \|, \forall \mathbf{x}, \mathbf{y} \in \mathbb{R}^d$.
\end{assumption}

This section presents a simple yet effective variant of SARAH to enable a larger step size. The improvement stems from making use of the data dependent $L_i$ in Assumption \ref{as.11}. The resultant algorithm that we term \textbf{D}ata \textbf{D}ependent \textbf{S}ARAH (D2S) is summarized in Alg. \ref{alg.3}. For simplicity D2S is developed based on SARAH, but it generalizes to L2S as well.

Intuitively, each $f_i$ provides a distinct gradient to be used in the updates. The insight here is that if one could quantify the ``importance'' of $f_i$ (or the gradient it provides), those more important ones should be used more frequently. Formally, our idea is to draw $i_t$ of outer loop $s$ according to a probability mass vector $\mathbf{p}_t^s \in \Delta_n$, where $\Delta_n:=\{\mathbf{p} \in \mathbb{R}_+^n | \langle \mathbf{1}, \mathbf{p} \rangle = 1 \}$. With $\mathbf{p}_t^s = 1/n$, D2S boils down to SARAH.
 
Ideally, finding $\mathbf{p}_t^s$ should rely on the estimation error as optimality crietrion. Specifically, we wish to minimize $\mathbb{E}[\| \mathbf{v}_t^s - \nabla F(\mathbf{x}_t^s)\|^2 | {\cal F}_{t-1}]$ in Lemma \ref{lemma.momt}. Writing the expectation explicitly, the problem can be posed as 
\begin{align}\label{eq.opt_p_prob}
	& ~~\min_{\mathbf{p}_t^s \in \Delta_n} \frac{1}{n^2} \sum_{i \in [n]} \frac{\| \nabla f_i(\mathbf{x}_t^s) - \nabla f_i(\mathbf{x}_{t-1}^s) \|^2 }{p_{t,i}^s} ~~~ \Rightarrow~~  (p_{t,i}^s)^* =  \frac{ \| \nabla f_i(\mathbf{x}_t^s) - \nabla f_i(\mathbf{x}_{t-1}^s) \|}{\sum_{j \in [n]} \| \nabla f_j(\mathbf{x}_t^s) - \nabla f_j(\mathbf{x}_{t-1}^s) \|} 
\end{align}
where the $(p_{t,i}^s)^*$ denotes the optimal solution. Though finding out $\mathbf{p}_t^s$ via \eqref{eq.opt_p_prob} is optimal, it is intractable to implement because $\nabla f_i(\mathbf{x}_{t-1}^s)$ and $\nabla f_i(\mathbf{x}_t^s)$ for all $i \in [n]$ must be computed, which is even more expensive than computing $\nabla F(\mathbf{x}_t^s)$ itself. However, \eqref{eq.opt_p_prob} implies that a larger probability should be assigned to those $\{f_i\}$ whose gradients on $\mathbf{x}_t^s$ and $\mathbf{x}_{t-1}^s$ change drastically. The intuition behind this observation is that a more abrupt change of the gradient suggests a larger residual to be optimized. Thus, $\|\nabla f_i (\mathbf{x}_t^s) - \nabla f_i(\mathbf{x}_{t-1}^s) \|^2$ in \eqref{eq.opt_p_prob} can be approximated by its upper bound $ L_i^2  \|\mathbf{x}_t^s  -  \mathbf{x}_{t-1}^s\|^2 $, which inaccurately captures gradient changes. The resultant problem and its optimal solution are
\begin{align}\label{eq.opt_p_sol2}
	& \min_{\mathbf{p}_t^s \in \Delta_n} \frac{1}{n^2} \sum_{i \in [n]} \frac{ L_i^2\| \mathbf{x}_t^s - \mathbf{x}_{t-1}^s \|^2 }{p_{t,i}^s} ~~~ 	\Rightarrow ~~~ 	(p_{t,i}^s)^* =  \frac{ L_i }{\sum_{j \in [n]} L_j}, \forall t, \forall s.
\end{align}
Choosing $\mathbf{p}_t^s$ according to \eqref{eq.opt_p_sol2} is computationally attractive not only because it eliminates the need to compute gradients, but also because $L_i$ is usually cheap to obtain in practice (at least for linear and logistic regression losses). Knowing $L = \max_{i \in [n]} L_i$ is critical for SARAH \citep{nguyen2017}; hence, finding $\mathbf{p}_t^s$ only introduces negligible overhead compared to SARAH. Accounting for $\mathbf{p}_t^s$, the gradient estimator $\mathbf{v}_t^s$ is also modified to an importance sampling based one to compensate for those less frequently sampled $\{f_i\}$
	\begin{equation}\label{eq.d2s_est}
		\mathbf{v}_t^s = \frac{\nabla f_{i_t} (\mathbf{x}_t^s ) -\nabla f_{i_t} (\mathbf{x}_{t-1}^s )}{n p_{t,i_t}^s} + \mathbf{v}_{t-1}^s.
	\end{equation}
Note that $\mathbf{v}_t^s$ is still biased, since $\mathbb{E} [ \mathbf{v}_t^s| {\cal F}_{t-1} ] \!= \!\nabla F(\mathbf{x}_t^s ) \!-\! \nabla F(\mathbf{x}_{t-1}^s )\!+\! \mathbf{v}_{t-1}^s \!\neq\! \nabla F(\mathbf{x}_t^s )$. As asserted next, with $\mathbf{p}_t^s$ as in \eqref{eq.opt_p_sol2} and $\mathbf{v}_t^s$ computed via \eqref{eq.d2s_est}, D2S indeed improves SARAH's convergence rate.
\begin{theorem}\label{thm.d2s}
	If Assumptions \ref{as.11}, \ref{as.2}, and \ref{as.3} hold, upon choosing $\eta < 1/ {\bar{L}}$ and a large enough $m$ such that $	\sigma_m:= \frac{1}{\mu\eta(m+1)} + \frac{\eta \bar{L}}{2- \eta \bar{L}} < 1$, D2S convergences linearly; that is, 
	\begin{equation*}
		\mathbb{E} \big[ \| \nabla F(\tilde{\mathbf{x}}_s) \|^2 \big] \leq  (\sigma_m)^s \| \nabla F(\tilde{\mathbf{x}}_0) \|^2, \forall s.
	\end{equation*}
\end{theorem}
Compared with SARAH's linear convergence rate $\tilde{\sigma}_m = \frac{1}{\mu\eta(m+1)} + \frac{\eta L}{2- \eta L}$ \citep{nguyen2017}, the improvement on the convergence constant $\sigma_m$ is twofold: i) if $\eta$ and $m$ are chosen the same in D2S and SARAH, it always holds that $\sigma_m \leq \tilde{\sigma}_m$, which implies D2S converges faster than SARAH; and ii) the step size can be chosen more aggressively with $\eta < 1/\bar{L}$, while the standard SARAH step size has to be less than $1/L$. The improvements are further corroborated in terms of the number of IFO calls, especially for ERM problems that are ill-conditioned.
\begin{corollary}\label{coro.d2s}
If Assumptions \ref{as.11}, \ref{as.2}, and \ref{as.3} hold, to find $ \tilde{\mathbf{x}}^s$ such that $\mathbb{E}\big[\| \nabla F(\tilde{\mathbf{x}}^s) \|^2\big] \leq \epsilon$, D2S requires ${\cal O}\big( (n+ \bar{\kappa}) \ln (1/\epsilon) \big)$ IFO calls, where $\bar{\kappa} := \bar{L}/\mu$.
\end{corollary}

\subsection{Optimal Solution of \eqref{eq.opt_p_prob}}
The optimal solution of \eqref{eq.opt_p_prob} can be directly obtained from the partial Lagrangian
\begin{equation*}
	{\cal L}(\mathbf{p}_t^s, \lambda) = \frac{1}{n^2} \sum_{i \in [n]} \frac{\| \nabla f_i(\mathbf{x}_t^s) - \nabla f_i(\mathbf{x}_{t-1}^s) \|^2 }{p_{t,i}^s} + \lambda \sum_{i \in [n]} p_{t,i}^s - \lambda.
\end{equation*}
Taking derivative w.r.t. $\mathbf{p}_t^s$ and set it to $\bm{0}$, we have 
\begin{equation*}
	p_{t,i}^s = \frac{\| \nabla f_i(\mathbf{x}_t^s) - \nabla f_i(\mathbf{x}_{t-1}^s) \|}{\sqrt{\lambda} n}.
\end{equation*}
Note that if $\lambda > 0$, it automatically satisfies $p_{t,i}^s \geq 0$. Then let $\sum_{i \in [n]} p_{t,i}^s = 1$, it is not hard to find the value of $\lambda$ and obtain \eqref{eq.opt_p_prob}. The solution of \eqref{eq.opt_p_sol2} can be derived in a similar manner.

\subsection{Proof of Theorem \ref{thm.d2s}}
The proof generalizes the original proof of SARAH for strongly convex problems \cite[Theorem 2]{nguyen2017}. Notice that the importance sampling based gradient estimator enables the fact $\mathbb{E}_{i_t} \big[ \mathbf{v}_t^s| {\cal F}_{t-1} \big] = \nabla F(\mathbf{x}_t^s ) -\nabla F(\mathbf{x}_{t-1}^s )+ \mathbf{v}_{t-1}^s$. By exploring this fact, it is not hard to see that the following lemmas hold. The proof has almost the same steps as those in \citep{nguyen2017}, except for the expectation now is w.r.t. a nonuniform distribution $\mathbf{p}_t^s$.

\begin{lemma}\label{lemma.copy3}
	\cite[Lemma 1]{nguyen2017} In any outer loop $s$, if $\eta \leq 1/L_F$, we have
	\begin{align*}
		\sum_{t=0}^m \mathbb{E} \big[ \|\nabla F(\mathbf{x}_t^s) \|^2 \big] \leq \frac{2}{\eta} \mathbb{E} \big[ F(\mathbf{x}_0^s) - F(\mathbf{x}^*)  \big] + \sum_{t=0}^m \mathbb{E} \big[ \| \nabla F(\mathbf{x}_t^s) - \mathbf{v}_t^s \|\big].
	\end{align*}	
\end{lemma}

\begin{lemma}\label{lemma.copy4}
The following equation is true
	\begin{align*}
		 \mathbb{E} \big[ \|\nabla F(\mathbf{x}_t^s)- \mathbf{v}_t^s \|^2 \big] = \sum_{\tau=1}^{t}\mathbb{E} \big[ \| & \mathbf{v}_\tau^s  - \mathbf{v}_{\tau-1}^s \|^2 	\big] -  \sum_{\tau=1}^t \mathbb{E} \big[ \|  \nabla F(\mathbf{x}_\tau^s)   -  \nabla F(\mathbf{x}_{\tau-1}^s) \|^2	\big].
	\end{align*}	
\end{lemma}

\begin{lemma}\label{lemma.thm1.lemma3}
	In any outer loop $s$, if $\eta$ is chosen to satisfy $1 - \frac{2}{\eta \bar{L}} < 0$, we have
	\begin{align*}
		 \mathbb{E} \big[  \big\|  \mathbf{v}_t^s -  \mathbf{v}_{t-1}^s\|^2 | {\cal F}_{t-1} \big] \leq \frac{ \eta \bar{L}}{2 - \eta \bar{L}}  \bigg( \|\mathbf{v}_{t-1}^s \|^2 -  \mathbb{E} \big[ \|\mathbf{v}_t^s \|^2 | {\cal F}_{t-1} \big]  \bigg), \forall t \geq 1.
	\end{align*}
\end{lemma}
\begin{proof}
	Consider that for any $t\geq 1$
	\begin{align*}
		& \quad ~ ~ \mathbb{E}_{i_t} \big[ \| \mathbf{v}_t^s  \|^2 | {\cal F}_{t-1} \big]  = \mathbb{E}_{i_t} \big[ \| \mathbf{v}_t^s -  \mathbf{v}_{t-1}^s  + \mathbf{v}_{t-1}^s \|^2 | {\cal F}_{t-1} \big] \nonumber \\
		& = \| \mathbf{v}_{t-1}^s \|^2 + \mathbb{E} \big[ \|  \mathbf{v}_t^s -  \mathbf{v}_{t-1}^s\|^2 | {\cal F}_{t-1} \big] + 2   \mathbb{E} \big[ \langle  \mathbf{v}_{t-1}^s  ,\mathbf{v}_t^s -  \mathbf{v}_{t-1}^s \rangle | {\cal F}_{t-1} \big] \nonumber \\
		& \stackrel{(a)}{=} \| \mathbf{v}_{t-1}^s \|^2 + \mathbb{E} \Big[ \|  \mathbf{v}_t^s -  \mathbf{v}_{t-1}^s\|^2  + \frac{2}{\eta}  \Big\langle  \mathbf{x}_{t-1}^s -  \mathbf{x}_t^s , \frac{\nabla f_{i_t} (\mathbf{x}_t^s ) -\nabla f_{i_t} (\mathbf{x}_{t-1}^s )}{n p_{t,i_t}^s}\Big\rangle \Big| {\cal F}_{t-1} \Big] \nonumber \\
		& \stackrel{(b)}{\leq} \| \mathbf{v}_{t-1}^s \|^2 + \mathbb{E} \Big[ \|  \mathbf{v}_t^s -  \mathbf{v}_{t-1}^s\|^2 - \frac{2}{\eta L_{i_t} n p_{t,i_t}^s}  \| \nabla f_{i_t} (\mathbf{x}_t^s ) -\nabla f_{i_t} (\mathbf{x}_{t-1}^s )  \|^2 \Big| {\cal F}_{t-1} \Big] \nonumber \\
		& \stackrel{(c)}{=} \| \mathbf{v}_{t-1}^s \|^2 + \mathbb{E} \Big[ \|  \mathbf{v}_t^s -  \mathbf{v}_{t-1}^s\|^2  - \frac{2 n p_{t,i_t}^s}{\eta L_{i_t} }  \|  \mathbf{v}_t^s -  \mathbf{v}_{t-1}^s\|^2 \Big| {\cal F}_{t-1} \Big] \nonumber \\
		& \stackrel{(d)}{=} \| \mathbf{v}_{t-1}^s \|^2 + \mathbb{E} \Big[  \Big(1 - \frac{2}{\eta \bar{L}} \Big)\|  \mathbf{v}_t^s -  \mathbf{v}_{t-1}^s\|^2 \Big| {\cal F}_{t-1} \Big]
	\end{align*}
	where (a) follows from \eqref{eq.d2s_est} and the update $\mathbf{x}_t^s = \mathbf{x}_{t-1}^s - \eta \mathbf{v}_t^s $; (b) is the result of \eqref{eq.apdx.smooth.2}; (c) is by the definition of $\mathbf{v}_t^s$; and (d) is by plugging \eqref{eq.opt_p_sol2} in. By choosing $\eta$ such that $1 - \frac{2}{\eta \bar{L}} < 0$, we have 
	\begin{align*}
		 \mathbb{E} \big[  \big\|  \mathbf{v}_t^s -  \mathbf{v}_{t-1}^s\|^2 | {\cal F}_{t-1} \big] \leq \frac{ \eta \bar{L}}{2 - \eta \bar{L}}  \bigg( \|\mathbf{v}_{t-1}^s \|^2 -  \mathbb{E} \big[ \|\mathbf{v}_t^s \|^2 | {\cal F}_{t-1} \big]  \bigg)
	\end{align*}
	which concludes the proof.
\end{proof}

\textbf{Proof of Theorem \ref{thm.d2s}:}
Using Lemmas \ref{lemma.copy4} and \ref{lemma.thm1.lemma3} we have 
\begin{align}\label{eq.thm.eq1}
	\mathbb{E} \big[ \|\nabla F(\mathbf{x}_t^s)- \mathbf{v}_t^s \|^2 \big] & = \sum_{\tau=1}^t \mathbb{E} \big[ \| \mathbf{v}_\tau^s  - \mathbf{v}_{\tau-1}^s \|^2 	\big]  -  \sum_{\tau=1}^t \mathbb{E} \big[ \|  \nabla F(\mathbf{x}_\tau^s)   -  \nabla F(\mathbf{x}_{\tau-1}^s) \|^2 	\big] 	\nonumber \\
	& \leq  \frac{ \eta \bar{L}}{2 - \eta \bar{L}}  \mathbb{E} \big[ \|\mathbf{v}_0^s \|^2 \big].
\end{align}
If we further let $\eta \leq 1/L_F$, plugging \eqref{eq.thm.eq1} into Lemma \ref{lemma.copy3}, we have
	\begin{align*}
		\sum_{t=0}^m \mathbb{E} \big[ \|\nabla F(\mathbf{x}_t^s) \|^2 \big] \leq \frac{2}{\eta} \mathbb{E} \big[ F(\mathbf{x}_0^s) - F(\mathbf{x}^*)  \big] + \frac{ (m+1) \eta \bar{L}}{2 - \eta \bar{L}}  \mathbb{E} \big[ \|\mathbf{v}_0^s \|^2 \big].
	\end{align*}	
	Since $\tilde{\mathbf{x}}^s$ is uniformly randomized chosen from $\{ \mathbf{x}_t^s\}_{t=0}^m$, by exploiting the fact $\mathbf{v}_0^s = \nabla F( \tilde{\mathbf{x}}^{s-1} )$ and $\mathbf{x}_0^s =  \tilde{\mathbf{x}}^{s-1} $, we have that 
	\begin{align}\label{eq.d2s_final}
		 \mathbb{E} \big[ \|\nabla F(\tilde{\mathbf{x}}^s) \|^2 \big] & \leq \frac{2}{\eta (m+1)} \mathbb{E} \big[ F( \tilde{\mathbf{x}}^{s-1}  ) - F(\mathbf{x}^*)  \big] + \frac{ \eta \bar{L}}{2 - \eta \bar{L}}  \mathbb{E} \big[ \|  \nabla F( \tilde{\mathbf{x}}^{s-1} ) \|^2 \big] \nonumber \\
		 & \leq  \bigg( \frac{2}{\mu \eta (m+1)} + \frac{ \eta \bar{L}}{2 - \eta \bar{L}} \bigg) \mathbb{E} \big[ \|\nabla F(\tilde{\mathbf{x}}^{s-1}) \|^2 \big] 
	\end{align}	
	where the last inequality follows from \eqref{eq.sca}. Unrolling $\mathbb{E} \big[ \|\nabla F(\tilde{\mathbf{x}}^{s-1}) \|^2 \big] $ in \eqref{eq.d2s_final}, Theorem \ref{thm.d2s} can be proved.

\subsection{Proof of Corollary \ref{coro.d2s}}
The proof is modified from \cite[Corollary 3]{nguyen2017}. By choosing $\eta = 0.5/(\bar{L})$ and $m = 4.5 \bar{\kappa}$, we have $\sigma_m$ in Theorem \ref{thm.d2s} bounded by
\begin{equation*}
	\sigma_m = \frac{1}{\frac{1}{2\bar{\kappa}}(4.5\bar{\kappa} + 1)} + \frac{0.5}{1.5} < \frac{7}{9}.
\end{equation*}
Then by Theorem \ref{thm.d2s}, by choosing $S$ as
\begin{equation*}
	S \geq \frac{\ln \big(  \| \nabla F( \tilde{\mathbf{x}}^0) \|^2 /\epsilon \big)}{\ln (9/7)} \geq \log_{7/9} (  \| \nabla F \big( \tilde{\mathbf{x}}^0)\|^2/\epsilon \big)
\end{equation*}
we have $\mathbb{E} \big[ \| \nabla F(\tilde{\mathbf{x}}^S) \|^2 \big] \leq (\sigma_m)^2 \| \nabla F \big( \tilde{\mathbf{x}}^0)\|^2 \leq  \epsilon$. Thus the number of IFO calls is 
\begin{align*}
	(n+2m) S = {\cal O} \big( (n+\bar{\kappa}) \ln (1/\epsilon)    \big).
\end{align*}

\section{Numerical Experiments}\label{apdx.tests}
Experiments for (strongly) convex cases are performed using python 3.7 on an Intel i7-4790CPU @3.60 GHz (32 GB RAM) desktop. The details of the used datasets are summarized in Table \ref{tab.dataset}. The smoothness parameter $L_i$ can be calculated via $L_i = \| \mathbf{a}_i\|^2/4$ by checking the Hessian matrix.

\begin{table}
\centering 
\caption{A summary of datasets used in numerical tests}\label{tab.dataset}
 \begin{tabular}{ c*{6}{|c}} 
    \hline
Dataset  & $d$  & $n$ (train)  & density & $n$ (test) & $L$ & $\lambda$
\\ \hline
\textit{a9a}  & $123$ & $32,561$ &  $11.28\%$ & $16,281$  & $3.4672$ & $0.0005$
\\ \hline
\textit{rcv1} &  $47,236$  & $20,242$ & $0.157\%$ &  $677,399$  & $0.25$  & $0.0001$
\\ \hline
\textit{w7a} &   $300$  & $24,692$ & $3.89\%$ & $25,057$ & $2.917$   & $0.005$
\\ \hline
\end{tabular} 
\end{table}

\textbf{L2S.} Since we are considering the convex case, we set $\lambda =0$ in \eqref{eq.test}. SVRG, SARAH and SGD are chosen as benchmarks, where SGD is modified with step size $\eta_k = 1/\big(\bar{L}(k+1)\big)$ on the $k$-th epoch. For both SARAH and SVRG, the length of inner loop is chosen as $m = n$. For a fair comparison, we use the same $m$ for L2S [cf. \eqref{eq.l2s_v}]. The step sizes of SARAH and SVRG are selected from $\{ 0.01/\bar{L}, 0.1/\bar{L}, 0.2/\bar{L}, 0.3/\bar{L}, 0.4/\bar{L}, 0.5/\bar{L}, 0.6/\bar{L}, 0.7/\bar{L}, 0.8/\bar{L}, 0.9/\bar{L}, 0.95/\bar{L}\}$ and those with best performances are reported. Note that the SVRG theory only effects when $\eta< 0.25/\bar{L}$. The step size of L2S is the same as that of SARAH for fairness.

\textbf{L2S-SC.} The parameters are chosen in the same manner as the test of L2S.

\textbf{L2S for on Nononvex Problems}
We perform classification on MNIST dataset using a $ 784 \times 128 \times 10$ feedforward neural network through Pytorch. The activation function used in hidden layer is sigmoid. SGD, SVRG, and SARAH are adopted as benchmarks. In all tested algorithms the batch sizes are $b=32$. The step size of SGD is ${\cal O}(\sqrt{b}/(k+1))$, where $k$ is the index of epoch; the step size is chosen as $b/(Ln^{2/3})$ for SVRG \citep{reddi2016}; and the step sizes are $\sqrt{b}/(2\sqrt{n}L)$ for SARAH \citep{nguyen2019} and L2S. The inner loop lengths are selected to be $m = n/b$ for SVRG and SARAH, while the same $m$ is used for L2S.

\end{document}